\documentclass{article}

\usepackage{amsmath,amssymb,amsthm}
\usepackage{arxiv}

\usepackage[utf8]{inputenc} 
\usepackage[T1]{fontenc}    
\usepackage{hyperref}       
\usepackage{url}            
\usepackage{booktabs}       
\usepackage{amsfonts}       
\usepackage{nicefrac}       
\usepackage{microtype}      
\usepackage{graphicx}
\usepackage{bm}
\usepackage{dsfont}
\usepackage[inline]{enumitem}

\DeclareMathOperator*{\argmin}{{\rm argmin}}

\newcommand{\cH}{\mathcal{H}}        
\newcommand{\cG}{\mathcal{G}}        
\newcommand{\cHstar}{\mathcal{H}^*}  
\newcommand{\cD}{\mathcal{D}}        
\newcommand{\cP}{\mathcal{P}}

\newcommand{\cN}{\mathcal{N}}

\newcommand{\vx}{\mathbf{x}}
\newcommand{\vv}{\mathbf{v}}
\newcommand{\vy}{\mathbf{y}}
\newcommand{\vz}{\mathbf{z}}
\newcommand{\vb}{\mathbf{b}}
\newcommand{\vs}{\mathbf{s}}
\newcommand{\bs}{\mathbf{s}}
\newcommand{\Dx}{\Delta \mathbf{x}}
\newcommand{\Dt}{\Delta t}
\newcommand{\dx}{d\mathbf{x}}

\newcommand{\R}{\mathbb{R}}          
\newcommand{\E}{\mathbb{E}}          
\newcommand{\Hstar}{H^*}

\newcommand{\Cov}{\mathrm{Cov}}
\newcommand{\Var}{\mathrm{Var}}

\newcommand{\vW}{\mathbf{W}}
\newcommand{\bpsi}{\bar\psi}
\newcommand{\sqnorm}[1]{\left\| #1 \right\|^2}

\newcommand{\SDSM}{\Delta\mathcal S_{\rm DSM}}    
\newcommand{\Sideal}{\Delta\mathcal S_{\rm ideal}} 

\newtheorem{theorem}{Theorem}[section]
\newtheorem{lemma}[theorem]{Lemma}
\newtheorem{proposition}[theorem]{Proposition}

\numberwithin{equation}{section} 

\title{The Loss Floor of Denoising Score Matching:\\ Fisher Geometry from Schr\"odinger Bridges}

\author{
Avinash~Raju \\
Great Wall Motors Co. Ltd.\\
Beijing, China \\
\texttt{avinashraju777@gmail.com} \\
  \And
Kai~Zhang \\
China Patent Information Center\\
Beijing, China \\
\texttt{rachelzhang261@gmail.com} \\
}

\begin{document}
	
\maketitle

\begin{abstract}
Denoising score matching trains diffusion models by regressing onto a conditional score, although the generative dynamics ultimately require the marginal score. The two objectives share the same population minimizer, but the conditional target remains random at fixed noisy state and therefore introduces an irreducible excess in the training loss. We isolate this excess and show that, for a general corruption kernel under mild regularity assumptions, it is exactly the trace of the Fisher--Rao metric of the conditional endpoint family, integrated along the diffusion trajectory. The result provides an exact conditional-variance decomposition of the denoising objective and identifies the information geometry recently observed in diffusion latent spaces as an intrinsic component of the training loss rather than an additional structure imposed on the model.
We derive the result from the Schr\"odinger bridge variational principle, within which the ideal objective arises as excess path-space relative entropy. For corruption diffusions, the Fisher term is proportional to the rate at which the noisy state loses mutual information about the clean data, separating the loss floor into an information flow determined by the data and a weight determined by the corruption schedule and training objective. In the Gaussian case this yields a closed form for the floor, recovers the reparametrization invariance of the continuous-time objective, and relates its high-SNR divergence to the information dimension of the data.
Finally, we examine consequences for practice. Raw training losses obtained with different noise ranges or weightings contain different additive floors and need not rank models consistently; subtracting the floor restores the correct ordering in our example. We also contrast the second-order geometry seen by the training objective with the third-order conditional statistics that enter numerical sampling error. Together, these results connect the variational origin of diffusion dynamics, the geometry of latent space, and the information-theoretic structure of denoising score matching.
\end{abstract}

\section{Introduction}
\label{sec:intro}

Denoising diffusion models \cite{Sohl-DicksteinW15, Ho2020DDPM, Song2020DDIM, Song2021ScoreBased, nichol2021improved} are now a standard framework for generative modeling, with applications spanning image synthesis \cite{Rombach2022LDM, Dhariwal2021, Karras2022ElucidatingTD}, video and audio generation \cite{Ho2022VideoDM, Blattmann2023AlignYL, Kong2020DiffWaveAV, Popov2021GradTTSAD}, molecular design \cite{Watson2023RFdiffusion, Zeni2025MatterGen}, and discrete domains such as language \cite{Austin2021D3PM, Lou2024SEDD, Nie2025LLaDA}. Their formulations are by now well understood from several equivalent viewpoints: time-reversed stochastic differential equations \cite{Anderson1982, Haussman86, Song2021ScoreBased}, variational objectives \cite{kingma2021variational, Huang21, Song2021Max}, flow matching and stochastic interpolants \cite{Lipman2022FlowMF, Liu2022FlowSA, Albergo2023StochasticIA}, and posterior denoising through Tweedie-type identities \cite{Robbins1992, Efron2011TweediesFA, Manor2023OnTP}. Common to all of these viewpoints is that training proceeds by denoising score matching (DSM) \cite{Hyvarinin2005, Vincent2011, Song2019}.

The generative dynamics, however, require the \emph{marginal} score $\vs(\vx,t)=\nabla_{\vx}\log P_t(\vx)$ of the noisy data distribution, which is intractable. DSM circumvents this by regressing onto the \emph{conditional} score $\vs_{\rm cond}(\vx,t;\vy) = \nabla_{\vx}\log q(\vx,t\mid\vy)$ of the corruption kernel for a clean sample $\vy$ \cite{Ho2020DDPM, Vincent2011, Song2019}. The conditional score is an unbiased estimator of the marginal score, so the two objectives share the same population minimizer; but at fixed $\vx$ the conditional target remains a random variable, and this sample-level stochasticity has a cost that the standard derivation leaves unquantified.

That cost is the subject of this work. Because sharing a minimizer does not imply sharing a loss value, replacing the marginal target by a random conditional target inflates the objective by a term that no model can reduce, and the question we answer is what this irreducible excess is and what it depends on.

Our main result is that the excess admits an exact geometric characterization. The fluctuation of the conditional score about its posterior mean is the score of the conditional endpoint distribution
\[
p(\vy\mid\vx,t)
=
\frac{P_{\rm data}(\vy)\,q(\vx,t\mid\vy)}
{P_t(\vx)}
\]
with respect to the latent coordinate $\vx$, so that its second moment is by definition the Fisher information of this family, the central object of information geometry \cite{BarndorffNielsen1978, amari2000methods, Nielsen2020}. This yields an exact orthogonal decomposition of the denoising objective,
\begin{equation}
	\SDSM
	=
	\Sideal
	+
	\frac12\int_0^T
	\gamma(t)\,
	\E_{\vx\sim P_t}
	\!\left[\operatorname{tr}g(\vx,t)\right]dt,
	\label{eq:intro_decomposition}
\end{equation}
where $\Sideal$ is the ideal marginal-score objective and $g(\vx,t)$ is the Fisher--Rao metric tensor of the conditional endpoint family. The second term depends only on the corruption and the data, never on the model: it is an irreducible loss floor intrinsic to conditional denoising score matching.

The Schrödinger bridge formulation \cite{schrodinger1931, schrodinger1932, Follmer1988, Csiszar1975, Leonard2014} provides a variational foundation for the reverse diffusion dynamics and supplies the common origin of both terms. A Schr\"odinger bridge is the relative-entropy projection of a reference path measure onto prescribed endpoint constraints, and its optimal dynamics are given by a Doob $h$-transform \cite{Doob1957}, in which the reference kernel is tilted by a positive harmonic function that enforces the terminal marginal; bridges and their stochastic-control interpretation have by now been used extensively in generative modeling \cite{Bortoli21, Wang21, Vargas21, ChenLiuTheodorou2022, Tong2023SimulationfreeSB, Shi2023DSBM}. Here the same variational structure plays a different role: the \emph{first} variation of the bridge functional produces the score-based generative drift and hence the ideal objective $\Sideal$, while the \emph{second} variation produces the quadratic form whose pullback to latent space is the Fisher metric appearing in the floor. For affine Gaussian corruption, $q(\vx,t\mid\vy) = \mathcal N(\vx;\alpha_t\vy,\sigma_t^2 I)$, this metric reduces to
\begin{equation}
	g_{ij}(\vx,t)
	=
	\frac{1}{\sigma_t^2}\delta_{ij}
	+
	\partial_i\partial_j\log P_t(\vx),
	\label{eq:intro_gaussian_metric}
\end{equation}
which is precisely the Hessian/Fisher geometry recently identified in diffusion latent spaces \cite{karczewski2026spacetime, lobashev2025hessian}. Equation~\eqref{eq:intro_decomposition} thus provides a variational origin for that geometry: it is not an additional structure imposed on the model, but the variance of the conditional regression target already present in the training loss.

The floor also admits an information-theoretic evaluation. For a general corruption diffusion $d\vx_t=f(\vx_t,t)\,dt+\sqrt{D(t)}\,d\vW_t$,
\begin{equation}
	\E_{P_t}\!\left[\operatorname{tr}g(\vx,t)\right]
	=
	-\frac{2}{D(t)}
	\frac{d}{dt}I(\vy;\vx_t),
	\label{eq:intro_infoflow}
\end{equation}
so the floor separates into a data-dependent information flow and a schedule-dependent weight; when the loss weighting coincides with the diffusion coefficient, the integral telescopes to a difference of mutual informations, and in the Gaussian case to a difference of differential entropies through the I--MMSE relation \cite{GuoShamaiVerdu2005}.

These identities have direct practical consequences. Because the additive floor varies with the SNR range, schedule, and loss weighting, raw training losses are not comparable across training configurations; we exhibit a ranking inversion that floor subtraction repairs. The factorization further separates what schedule design can and cannot affect, complementing recent work on information-based scheduling and noise allocation \cite{kingma2021variational, Stancevic2025Entropic, Kong2026Schedule, CosineFisherRao2025, Sabour2024AYS, Williams2024ScoreOptimal, Ambrogioni2026Atoms}. Finally, the training objective probes only the second cumulant of the conditional endpoint law, whereas numerical integration of the generative dynamics also probes the third; we use this contrast as a diagnostic for where third-order geometric effects become significant.

The paper is organized around the three identities displayed above: the loss floor is the integrated trace of a Fisher--Rao metric (Eq.~\eqref{eq:intro_decomposition}); that trace is the rate of mutual-information loss along the corruption (Eq.~\eqref{eq:intro_infoflow}); and for Gaussian corruption the metric and the floor admit the closed forms that underlie the geometry of Eq.~\eqref{eq:intro_gaussian_metric}. Section~\ref{sec:bridge} derives the ideal objective from the Schr\"odinger bridge variational principle and reduces it to DSM. Section~\ref{sec:fisher} constructs the Fisher geometry of the bridge and proves the decomposition theorem. Section~\ref{sec:infoflow} evaluates the floor in information-theoretic terms. Section~\ref{sec:consequences} develops the consequences for training and sampling, and Section~\ref{sec:discussion} discusses scope, extensions, and limitations. Rather than introducing a new diffusion objective or scheduling algorithm, this work exposes an exact decomposition already implicit in the standard objective and develops the geometric and information-theoretic structure it entails.

\section{From Path-Space Entropy to Denoising Score Matching}
\label{sec:bridge}

This section derives the training objective of diffusion models from the Schr\"odinger bridge variational principle. Our presentation proceeds in four steps: 
\begin{enumerate*}[label=(\roman*)]
\item the bridge problem selects an optimal path measure by relative-entropy minimization and its optimality conditions take the form of a Doob $h$-transform
\item the excess path entropy relative to the optimum reduces to a quadratic score objective,
\item and a conditional decomposition of the score field reduces that objective to denoising score matching.
\end{enumerate*}
The derivation below is included primarily to fix the variational origin of the score-matching objective. Its role is modest but useful as it provides a direct route from path-space entropy to the marginal-score objective against which denoising score matching will later be compared. The problem of reverse diffusion was studied from an action principle perspective in \cite{Premkumar2023GenerativeDF} whose notations we borrow.

\subsection{The path-space variational principle}
\label{subsec:variational}

Our goal is to transform an initial distribution $P_0=P_{\rm prior}$ into a target distribution $P_T=P_{\rm data}$ through a stochastic process. Let $\cG$ denote a reference path measure, assumed Markovian and governed by killed forward and backward Kolmogorov equations in its final and initial arguments, respectively. The Schr\"odinger bridge problem \cite{schrodinger1931, schrodinger1932, Leonard2014} arises from a large-deviation question: among the empirical path measures generated by many independent realizations of $\cG$, which one dominates when one conditions on the endpoint marginals being $P_0$ and $P_T$?

By Sanov's theorem, which assigns to the empirical measure of independent samples a large-deviation rate given by relative entropy \cite{DemboZeitouni1998, csiszar1984information}, the empirical path measure $\cH$ asymptotically carries the large-deviation weight
\[ \mathbb P(\cH_n\simeq\cH) \asymp \exp\left[-nD_{\rm KL}^{\rm path}(\cH\,\|\,\cG)\right], \]
so that conditioning on the endpoint constraints restricts the admissible measures to $\mathcal C=\{\cH:H_0=P_0,\;H_T=P_T\}$ and the dominant contribution is the constrained minimum of the rate function,
\begin{equation} \label{eq:SB_problem}
\begin{split}
	\cHstar &= \argmin_{\cH:\, H_0=P_0,\;H_T=P_T}  D^{\rm path}_{KL}(\cH \| \cG),\\
	D^{\rm path}_{KL}(\cH \| \cG) &= \int d\vx_{0:N} \, \cH(\vx_{0:N}) \ln \frac{\cH(\vx_{0:N})}{\cG(\vx_{0:N})}.
\end{split}
\end{equation}
Here the interval $[0,T]$ is discretized into $N\gg1$ segments of length $\Dt=T/N$, and the reference measure factorizes as
\begin{equation}
	\cG(\vx_{0:N}) = P_0(\vx_0) \prod_{k=0}^{N-1} G(\vx_{k+1} \mid \vx_k).
\end{equation}
Note that $\cG$ is required neither to be a probability measure nor to map $P_0$ to $P_T$, since the endpoint constraints are imposed variationally rather than by construction of the reference.

Solving Eq.~\eqref{eq:SB_problem} by Lagrange multipliers (Appendix~\ref{app:schrodinger_bridge_variation}) yields the factorized optimum
\begin{equation} \label{eq:sb_factorized_sol}
	\cHstar(\vx_{0:N}) = f(\vx_0)\,\cG(\vx_{0:N})\,g(\vx_N),
\end{equation}
where $f$ and $g$ depend only on the initial and final endpoints. This is precisely the structure of a Doob $h$-transform \cite{Doob1957}, obtained here as a direct consequence of the variational principle: the optimal path measure is the reference diffusion biased at its endpoints by factors that enforce the marginal constraints. Since $\cHstar$ is itself Markovian, so that
\begin{equation}
	\cH(\vx_{0:N}) = P_0(\vx_0) \prod_{k=0}^{N-1} H(\vx_{k+1} \mid \vx_k),
\end{equation}
the path-space divergence in Eq.~\eqref{eq:SB_problem} decomposes into a sum of per-step divergences between transition kernels, a fact we exploit below.

\subsection{Optimal bridge dynamics}
\label{subsec:doob}

The bridge dynamics at intermediate times are most transparently expressed through two auxiliary fields: the backward field $\psi(\vx,t)$ propagates the terminal constraint backward in time, while the forward field $\bpsi(\vx,t)$ propagates the initial condition forward. A direct consequence of the factorized form Eq.~\eqref{eq:sb_factorized_sol} is that the intermediate marginal is their product,
\begin{equation} \label{eq:product_form}
	P(\vx, t) = \psi(\vx, t) \, \bpsi(\vx, t),
\end{equation}
so the density at each point of spacetime is the confluence of information flowing from the past and the future. In the continuum limit the fields obey a dual pair of linear equations (Appendix~\ref{app:schrodinger_bridge_variation}), namely the adjoint (backward) Kolmogorov equation
\begin{equation}
	\partial_t \psi + \vv \cdot \nabla \psi + \frac{\gamma(t)}{2} \nabla^2 \psi - V_G \psi = 0,
	\label{eq:psi-eom}
\end{equation}
and the forward Fokker--Planck equation
\begin{equation}
	\partial_t \bpsi + \nabla \cdot (\vv \bpsi) - \frac{\gamma(t)}{2} \nabla^2 \bpsi + V_G \bpsi = 0,
	\label{eq:bpsi-eom}
\end{equation}
where $\vv$ is the drift of the reference process, $\gamma(t)$ its diffusion coefficient, and $V_G(\vx,t)$ a possible killing potential.

The optimal transition kernel is the reference kernel tilted by the ratio of backward fields,
\begin{equation} \label{eq:tilted_kernel}
	\Hstar(\vx_{k+1} \mid \vx_k) = \frac{\psi_{k+1}(\vx_{k+1})}{\psi_k(\vx_k)}\,G(\vx_{k+1} \mid \vx_k),
\end{equation}
so that the bridge kernel is a local Doob tilt of the reference kernel. Its short-time expansion (Appendix~\ref{app:schrodinger_bridge_variation}), which requires retaining second-order terms in $\Dx$ because $\Dx\sim\sqrt{\Dt}$ under the reference measure, preserves the diffusion coefficient and shifts the drift:
\begin{equation}
	\vv_{*}(\vx, t) = \vv(\vx, t) + \gamma(t)\, \nabla \ln \psi(\vx, t).
	\label{eq:optimal_drift}
\end{equation}
The correction $\gamma\nabla\ln\psi$ is the control field by which the bridge steers the reference dynamics toward the terminal constraint.

A particularly tractable sector of this construction underlies score-based generative modeling. If the reference process satisfies $V_G = -\nabla\cdot\vv$, the forward equation admits the trivial solution $\bpsi = {\rm const.}$, which we set to unity using the scaling symmetry of Eq.~\eqref{eq:sb_factorized_sol}; the entire burden of the terminal constraint is then carried by the backward field $\psi$, and $P_t=\psi_t$ by Eq.~\eqref{eq:product_form}, so that the optimal drift Eq.~\eqref{eq:optimal_drift} reduces to the score-matching drift
\begin{equation}
	\vv_{*}(\vx, t) = \vv(\vx, t) + \gamma(t)\, \nabla \ln P_t(\vx).
	\label{eq:score_drift}
\end{equation}
The learned score is therefore not an arbitrary regression target but the optimal control field required to satisfy the terminal data constraint.\footnote{Here $t$ parameterizes the generative direction, so the sign convention of Eq.~\eqref{eq:optimal_drift} differs from conventional reverse-time SDE notation; the two are related by the corresponding change of orientation.}

\subsection{Excess path entropy and the ideal score objective}
\label{subsec:excess}

We now turn to learning the score $\vs=\nabla\ln P_t$ from a finite dataset $\cD=\{\vy_i\}_{i=1}^N\sim P_{\rm data}$. The optimum $\cHstar$ is the $I$-projection of $\cG$ onto the constraint set \cite{Csiszar1975, csiszar1984information}, for which relative entropy obeys the Pythagorean identity
\begin{equation} \label{eq:pythagorean_identity}
	D^{\text{path}}_{KL}(\cH \| \cG) = D^{\text{path}}_{KL}(\cHstar \| \cG) + D^{\text{path}}_{KL}(\cH \| \cHstar)
\end{equation}
for any feasible $\cH$ sharing the endpoint marginals of $\cHstar$. Since the first term on the right is independent of $\cH$, learning the bridge is equivalent to minimizing the excess action
\begin{equation}
	\Delta \mathcal S \;\equiv\; D_{KL}^{\text{path}}(\cH \| \cG) - D_{KL}^{\text{path}}(\cHstar \| \cG) \;=\; D_{KL}^{\text{path}}(\cH \| \cHstar) \;\geq\; 0.
	\label{eq:excess_action}
\end{equation}
For Markovian path measures with a common initial distribution, this path-space divergence decomposes into per-step kernel divergences, each of which is quadratic in the drift difference to leading order in $\Dt$ (Appendix~\ref{app:short-time-expansion}), and summing over time slices gives
\begin{equation}
	\Delta \mathcal S = \frac{1}{2} \int_0^T \frac{dt}{\gamma(t)} \, \E_{\vx \sim P_t^H} \left[ \| \vv_H(\vx, t) - \vv_*(\vx, t) \|^2 \right],
	\label{eq:excess_action_integral}
\end{equation}
where the expectation is taken under $P_t^H$, the marginal of the \emph{model's own} dynamics. This on-policy weighting is computationally prohibitive, since it requires simulating the generative process during training. However, at the level of unrestricted drift fields the integrand is point-wise minimized by $\vv_H=\vv_*$, so the functional minimizer is independent of the positive weighting measure, and we therefore replace $P_t^H$ by the reference noising marginal $P_t$, which is available without simulation, obtaining the simulation-free objective \cite{Tong2023SimulationfreeSB}
\begin{equation}
	\Delta \mathcal S_{\text{SM}} = \frac{1}{2} \int_0^T \frac{dt}{\gamma(t)} \, \E_{\vx \sim P_t} \left[ \| \vv_H(\vx, t) - \vv_*(\vx, t) \|^2 \right].
	\label{eq:score_matching_objective_continuous}
\end{equation}
For a restricted parametric model this replacement does affect the projection, a point we return to in Sec.~\ref{sec:discussion}.

To expose the quantity being learned, we parameterize the candidate drift as a control deformation of the reference,
\begin{equation}
	\vv_H(\vx,t) = \vv(\vx,t)+\gamma(t)\,\mathbf{s}_\theta(\vx,t),
\end{equation}
mirroring the structure of the optimal drift Eq.~\eqref{eq:optimal_drift}. The reference drift then cancels from the difference $\vv_H-\vv_*$---the step by which the reference dynamics disappears from the objective---leaving the ideal score-matching loss
\begin{equation} \label{eq:score_function_learning_obj}
	\Sideal \;=\; \frac12 \int_0^T dt\,\gamma(t)\, \E_{\vx\sim P_t} \left[ \sqnorm{\mathbf{s}_\theta(\vx,t) - \nabla_\vx\log P_t(\vx) } \right],
\end{equation}
where we have used $P_t=\psi_t$ in the score-matching sector. Equation~\eqref{eq:score_function_learning_obj} is \emph{ideal} in the sense that its regression target, the marginal score, is not available in closed form.

\subsection{Denoising score matching as conditional regression}
\label{subsec:dsm}

The ideal objective depends on the marginal score, whose evaluation requires the global field $\psi$. Because $\psi$ obeys a \emph{linear} equation, the principle of superposition allows us to define a conditional field $\psi^{(\vy)}(\vx,t)$ corresponding to a bridge pinned to the Dirac terminal constraint $g(\vx_T)=\delta^d(\vx_T-\vy)$. The global field is then the continuous superposition of these conditional fields over the data distribution,
\begin{equation}
\psi(\vx, t) = \int d\vy \, P_{\rm data}(\vy) \, \psi^{(\vy)}(\vx, t) = \E_{\vy \sim P_{\rm data}}\!\left[\psi^{(\vy)}(\vx, t)\right] = P(\vx, t).
\end{equation}
Within the score-matching sector, each conditional field is the corruption kernel itself,
\begin{equation}
\psi^{(\vy)}(\vx, t) = q(\vx, t \mid \vy) = \frac{\cHstar(\vx, \vy)}{P_{\rm data}(\vy)}.
\end{equation}
Replacing the data expectation by its Monte Carlo estimate and applying Jensen's inequality to the convex squared norm then defines the tractable upper bound
\begin{align}
\SDSM &:= \frac12 \int_0^T dt\,\gamma(t) \, \E_{\vy \sim P_{\rm data}} \E_{\vx \sim q(\vx, t\mid \vy)} \left[ \sqnorm{\mathbf{s}_\theta(\vx,t) - \nabla_\vx\ln  q(\vx, t \mid \vy) } \right] \nonumber \\
&\geq \frac12 \int_0^T dt\,\gamma(t) \, \E_{\vx \sim P_t} \left[ \sqnorm{\mathbf{s}_\theta(\vx,t) - \nabla_\vx\ln \left( \E_{\vy \sim P_{\rm data}}[\psi^{(\vy)}(\vx, t)] \right) } \right] = \Sideal .
\label{eq:final_score_matching_loss}
\end{align}
Every term in $\SDSM$ is computable: training samples a data point $\vy$, corrupts it through the reference kernel, and regresses $\mathbf{s}_\theta(\vx,t)$ onto the conditional score $\nabla_\vx \ln q(\vx, t \mid \vy)$, which for affine Gaussian corruption is linear in the noise and reduces Eq.~\eqref{eq:final_score_matching_loss} to the standard denoising score-matching loss of \cite{Vincent2011, Song2019, Ho2020DDPM}.

The Jensen gap $\SDSM-\Sideal$ is the price of this substitution. Since the conditional score is an unbiased estimator of the marginal score, the gap does not shift the minimizer, but it does shift the value of the objective by an additive, model-independent term; the remainder of the paper studies that term---what it is, what geometry it carries, and what it implies for practice.

\section{The Fisher Geometry of the Denoising Loss}
\label{sec:fisher}

Section~\ref{sec:bridge} obtained the generative dynamics from the \emph{first} variation of a path-space relative entropy. We now show that the \emph{second} variation of the same functional determines the geometry of the states the diffusion visits, and that this geometry is exactly what inflates the denoising objective above its ideal value. The construction proceeds in three steps:
\begin{enumerate*}[label=(\arabic*)]
\item the second variation defines a canonical quadratic form on the bridge family, measurable with respect to the endpoints alone (Sec.~\ref{subsec:endpoint_geometry}),
\item  the requirements of a regular, computable corruption select a finite-dimensional reduction of that form (Secs.~\ref{subsec:regular_family}--\ref{subsec:pullback}), and
\item the Jensen gap of Sec.~\ref{subsec:dsm} is identified with its integrated trace (Sec.~\ref{subsec:floor}).
\end{enumerate*}
Derivations are collected in Appendix~\ref{app:latent}.

\subsection{Endpoint geometry of the bridge}
\label{subsec:endpoint_geometry}

Every solution of the bridge problem is an endpoint tilt of the reference:
\begin{equation} \label{eq:tilt_family}
	\ln \cH^{(u,v)}(\vx_{0:N}) = u(\vx_0) + v(\vx_N) + \ln \cG(\vx_{0:N}) - \Lambda[u,v],
\end{equation}
with $\Lambda[u,v] = \ln\int \dx_{0:N}\,\cG\,e^{u(\vx_0)+v(\vx_N)}$ finite on a domain $\cD$, and $\rho^{(u,v)}$ the joint endpoint law under $\cH^{(u,v)}$.

\begin{lemma}[Ambient Fisher form]\label{lem:ambient_fisher}
	For $(u,v)$ interior to $\cD$,
	\begin{equation} \label{eq:ambient_fisher}
		\delta^2\Lambda\big[(\delta u,\delta v)\big] \;=\; \Var_{\rho^{(u,v)}}\big[\delta u(\vx_0)+\delta v(\vx_N)\big].
	\end{equation}
\end{lemma}

The proof observes that $\Lambda$ is the cumulant-generating functional of the endpoint evaluations under $\cG$ (Appendix~\ref{app:latent}). For the Schr\"odinger bridge problem $\cG$ need not be a probability measure, owing to the killing potential of Sec.~\ref{sec:bridge}, but for the exponential-family argument it suffices that $\cG$ be $\sigma$-finite with the relevant exponential moments, which places the bridge within the framework of infinite-dimensional exponential families \cite{PistoneSempi1995, CenaPistone2007}.

Expanding Eq.~\eqref{eq:ambient_fisher} yields
\begin{equation} \label{eq:three_terms}
\delta^2\Lambda = \Var_\rho[\delta u(\vx_0)] + \Var_\rho[\delta v(\vx_N)] + 2\,\Cov_\rho[\delta u(\vx_0),\delta v(\vx_N)],
\end{equation}
so the two endpoint sectors are \emph{not} orthogonal in the ambient geometry. In short, the cross term carries the endpoint dependence induced by the reference dynamics and the bridge constraints.

A deeper implication of Eq.~\eqref{eq:ambient_fisher} concerns the tangent space rather than the metric itself. For an infinitesimal deformation of the bridge family, the first-order variation of the log-likelihood, i.e. the tangent vector to the statistical manifold at $\cH^{(u,v)}$, is
\[
\delta\ell = \delta u(\vx_0) + \delta v(\vx_N) - \E[\delta u(\vx_0)+\delta v(\vx_N)],
\]
so \emph{every infinitesimal likelihood ratio within the bridge family is measurable with respect to the endpoint $\sigma$-algebra}. The endpoint pair thus constitutes a sufficient statistic for the local statistical experiment defined by the tilt family, and Eq.~\eqref{eq:ambient_fisher} is its second moment. This places the construction within the classical framework in which Fisher information is monotone under Markov kernels and preserved exactly under sufficient reductions \cite{Chentsov1982, AyJostLeSchwachhofer2017}.

\subsection{A regular conditional endpoint family}
\label{subsec:regular_family}

A natural first candidate for a latent geometry is to condition the path measure on an interior state $X_t=\vx$ and to regard $\vx\mapsto\cH(\cdot\mid X_t=\vx)$ as a statistical family indexed by the spatial coordinate. This construction, however, is not regular. For distinct $\vx\neq\vx'$, the conditioned measures are supported on the disjoint path sets $\{X_t=\vx\}$ and $\{X_t=\vx'\}$ and are therefore mutually singular, so that no finite local Kullback--Leibler expansion exists from which a Fisher metric could be extracted. This obstruction is not a pathology of the bridge formulation but the generic consequence of exact point conditioning for continuous-path measures \cite[see, e.g., Section~2.2]{Kallenberg2002}, and it rules out exact conditioned path measures as a regular Fisher family, motivating instead a reduction that retains the endpoint variables.

The sufficiency of the endpoint tangent space suggests such a reduction: retain the endpoint experiment and condition \emph{it} on the latent state,
\begin{equation} \label{eq:endpoint_immersion}
	\iota_t:\;\vx \;\longmapsto\; p(\vx_0,\vx_N \mid X_t = \vx),
\end{equation}
defining a chain of reductions $\cP_{\rm path}\to\cP_{\rm end}\to\cP_{\rm end\mid lat}$, of which the first step is justified by the endpoint measurability of the bridge tangent space and the second is necessitated by the requirement of a regular parameterization. Equation~\eqref{eq:endpoint_immersion} discards precisely the path information on which the ambient Fisher form does not depend, and while we do not claim that this reduction is unique, it is a minimal reduction consistent with the tangent structure, and it is the one the denoising objective itself uses.

\subsection{Gaussian corruption from structural requirements}
\label{subsec:gaussian_realization}

Equation~\eqref{eq:endpoint_immersion} is still completely general, and the diffusion construction used in practice is obtained by imposing what a trainable model requires:
\begin{enumerate}[label=(\roman*),ref=(\roman*)]
	\item \label{item:req-dataspace} the latent state lives in the data space $\R^d$ rather than a learned code space;
	\item \label{item:req-hierarchy} the states form an ordered Markov hierarchy from the data to a fixed, data-independent prior left invariant by the transition;
	\item \label{item:req-corruption} the corruption is specified before training (hence state-independent), isotropic, and continuous in the continuum limit.
\end{enumerate}
The ordering already constrains the information flow. Along the corruption chain, which runs from the data endpoint $\vx_N$ toward the prior endpoint $\vx_0$, the data-processing inequality gives $I(\vx_N;\vx_{k}) \le I(\vx_N;\vx_{k+1})$, and invariance of the prior under the transition gives $D_{KL}(P_{k}\|P_{\rm prior}) \le D_{KL}(P_{k+1} \|P_{\rm prior})$ whenever the step from $k{+}1$ to $k$ moves along the chain, so that information about the data can only be discarded, never created, as corruption proceeds \cite{CoverThomas2006}.

A state-independent, isotropic, continuous Markov corruption is an additive diffusion $d\vx_t = \vb(t)dt + \sqrt{\gamma(t)}\,d\vW_t$, whose finite-time kernel is Gaussian,
\begin{equation} \label{eq:gaussian_encoder}
	q(\vx,t\mid \vx_N) = \cN\big(\vx;\ \alpha_t \vx_N,\ \sigma_t^2 \mathbf I\big),
\end{equation}
after absorbing the deterministic drift into $\alpha_t$. We present Eq.~\eqref{eq:gaussian_encoder} not as a uniqueness theorem for all corruption processes but as the natural continuous realization of requirements \eqref{item:req-dataspace}--\eqref{item:req-corruption} and the standard family used in practice \cite{Ho2020DDPM, Song2021ScoreBased, kingma2021variational}; relaxing state independence, isotropy, or continuity leads to broader model classes discussed in Sec.~\ref{sec:discussion}.

Read as a function of the clean endpoint, Eq.~\eqref{eq:gaussian_encoder} exhibits a finite sufficient statistic:
\begin{equation} \label{eq:ref_exp_split}
	\ln q(\vx,t\mid\vx_N) = A(\vx,t)\cdot\bs(\vx_N) + B(\vx,t) + C(\vx_N),
	\qquad
	\bs(\vx_N) = \Big(\vx_N,\, -\tfrac12\sqnorm{\vx_N}\Big),
\end{equation}
with $d+1$ components: $d$ first moments coupling to the spatial latent coordinates and one second moment coupling to the noise coordinate. The extra dimension of latent spacetime is thus the second-moment statistic of the corruption. Isotropy is what keeps this statistic one-dimensional: an anisotropic Gaussian requires the full quadratic $-\tfrac12\vx_N\otimes\vx_N$, giving $m = d + d(d{+}1)/2$ natural coordinates for $d+1$ latent coordinates, so the latent family would acquire codimension and become a \emph{curved} exponential family \cite{Efron1975, Amari1982}. Absorbing the $\vx$-independent factors into a carrier $\nu = e^{C}g$ yields
\begin{equation} \label{eq:p_g}
	p(\vx_N\mid\vx,t) = \exp\big[A(\vx,t)\cdot\bs(\vx_N) - \Lambda(A(\vx,t))\big]\,\nu(\vx_N),
\end{equation}
in which the learned model enters only through $\nu$, while the map to natural coordinates is fixed by the corruption. Explicitly, with $\mathrm{SNR}_t = \alpha_t^2/\sigma_t^2$ and $\tilde\vx = \vx/\alpha_t$,
\begin{equation} \label{eq:cone}
	A(\vx,t) \;=\; \Big(\frac{\alpha_t}{\sigma_t^2}\vx,\ \frac{\alpha_t^2}{\sigma_t^2}\Big) \;=\; \mathrm{SNR}_t\cdot\big(\tilde\vx,\,1\big).
\end{equation}
Latent spacetime is therefore a cone in natural-parameter space: the direction of $A$ is the rescaled spatial coordinate and its magnitude is the signal-to-noise ratio. For a nondegenerate schedule, the Jacobian of $A$ has rank $d+1$, so the latent coordinates locally parameterize an open subset of natural-parameter space and inherit its dual affine structure, while a monotone reparametrization of time slides along the ray Eq.~\eqref{eq:cone} without changing it---the geometric content of the schedule reparametrization invariance of diffusion objectives \cite{kingma2021variational}, to which we return in Sec.~\ref{sec:infoflow}.

\subsection{The pullback metric}
\label{subsec:pullback}

The immersion $\iota_{\rm lat}:(\vx,t)\mapsto A(\vx,t)$ carries the ambient form to latent spacetime, and for coordinates $z^\mu,z^\nu\in\{x^1,\dots,x^d,t\}$ the induced metric factorizes as
\begin{equation} \label{eq:factorization}
	g_{\mu\nu} \;=\; \underbrace{\partial_\mu A^a\,\partial_\nu A^b}_{\text{fixed by the schedule}}\; \underbrace{\Cov_{p(\vx_N\mid\vx,t)}\big[\bs_a,\bs_b\big]}_{\text{data and model}}.
\end{equation}
Once the corruption and its sufficient statistic are fixed, the metric follows from the second variation, and the two factors separate cleanly due to state independence of the corruption.

\begin{proposition}[Endpoint split]\label{prop:pullback}
	The two endpoints are conditionally independent given the latent state, $p(\vx_0,\vx_N\mid X_t=\vx) = p(\vx_0\mid X_t=\vx)\,p(\vx_N\mid X_t=\vx)$, and consequently
	\begin{equation} \label{eq:tangent_chain}
		g_{\mu\nu} = g^{\rm past}_{\mu\nu} + g^{\rm future}_{\mu\nu}.
	\end{equation}
\end{proposition}
The endpoint sectors are not orthogonal in the ambient metric Eq.~\eqref{eq:three_terms}; the cross term vanishes only after conditioning on the intermediate state, by the Markov property. The past sector $g^{\rm past}$ is the Fisher information of the seed posterior $p(\vx_0\mid\vx,t)$, measuring how much the corrupted state reveals about its initialization, and it is generically nonzero yet invisible to denoising score matching, whose regression target involves the clean endpoint alone.

\begin{proposition}[The future sector is the loss metric]\label{prop:tangent}
	For the corruption Eq.~\eqref{eq:gaussian_encoder},
	\begin{equation} \label{eq:g_chi}
	\begin{split}
		g^{\rm future}_{ij}(\vx,t) &= \Cov_{p(\vx_N\mid\vx,t)}\big[\partial_i \ln q,\ \partial_j \ln q\big] = \frac{\alpha_t^2}{\sigma_t^4}\Cov\big[x_N^i,x_N^j \mid X_t=\vx\big] \\
		&= \frac{1}{\sigma_t^2}\delta_{ij} + \partial_i\partial_j \ln P_t(\vx).
	\end{split}
	\end{equation}
\end{proposition}
Equation~\eqref{eq:g_chi} is the tensor whose trace appears in the decomposition theorem below. This metric has been studied descriptively in recent work \cite{karczewski2026spacetime, lobashev2025hessian} and the present derivation shows that it follows from the second variation of the bridge objective. It is worth emphasizing two aspect of this construction. First, the full latent metric is Eq.~\eqref{eq:tangent_chain}, while the training objective sees only its future sector, so the geometry is not an ad hoc construction reverse-engineered from the loss but contains strictly more structure than the loss probes. Second, the cone structure makes the information hierarchy quantitative. Writing $K(\theta) = D_{KL}(p_\theta\|\nu)$ for the information the latent carries about the clean endpoint, the exponential-family identity $\nabla_\theta K = g(\theta)\theta$ gives, along the ray $\theta = r\,n$ of Eq.~\eqref{eq:cone},
\begin{equation} \label{eq:radial_info}
	\frac{dK}{dr} \;=\; r\,g_{rr}, \qquad g_{rr} = n^{\top}\Cov_{p}[\bs]\,n .
\end{equation}
Since the radial coordinate is the signal-to-noise ratio, the information the latent retains about the data decays along the corruption at a rate set by the radial component of the Fisher metric, which is the pointwise counterpart of the integrated identity of Sec.~\ref{sec:infoflow} and a first indication that the loss floor and an accumulated information are the same quantity rather than merely equal numbers.

The construction of this section assembles classical components. The identification of the Fisher metric with the local Kullback--Leibler form and its monotonicity under Markov kernels are due to Chentsov \cite{Chentsov1982, AyJostLeSchwachhofer2017}, the endpoint-tilt family is an instance of the infinite-dimensional exponential families of Pistone and Sempi \cite{PistoneSempi1995, CenaPistone2007}, the induced geometry of a finite-dimensional submanifold is the theory of curved exponential families of Efron and Amari \cite{Efron1975, Amari1982}, and the reparametrization invariance in the SNR coordinate is standard \cite{kingma2021variational}. What is new here is the assembly: the ambient form is fixed by the bridge variational principle, and the latent metric is a pullback of that fixed form rather than a separate postulate.

\subsection{The irreducible loss floor}
\label{subsec:floor}

It remains to identify the geometry of Secs.~\ref{subsec:endpoint_geometry}--\ref{subsec:pullback} with the Jensen gap of Sec.~\ref{subsec:dsm}. Denote by $q(\vy \mid \vx, t) = P_{\rm data}(\vy)\, q(\vx, t \mid \vy)/P(\vx,t)$ the posterior over clean data given the noisy observation, and write $\nabla \ln q(\vy \mid \vx, t)$ for its score in $\vx$. Expanding the squared norm in Eq.~\eqref{eq:final_score_matching_loss} around the ideal target splits the gap into two contributions:
\begin{align}
	\begin{split}
		\SDSM - \Sideal &= \frac12 \int_0^T dt\,\gamma(t)\, \E_{\vx \sim P_t} \E_{\vy \sim q(\vy \mid \vx, t)}\Big\{ \\
		&\qquad \sqnorm{\nabla \ln  q(\vy \mid \vx, t)} - 2\,\nabla \ln  q(\vy \mid \vx, t) \cdot \big(\mathbf{s}_{\theta} - \nabla \ln  P(\vx, t) \big)  \Big\}.
	\end{split}
	\label{eq:gap_split}
\end{align}
The cross term vanishes identically for \emph{every} $\theta$, not merely at convergence: by the score identity
\begin{equation}
	\E_{\vy \sim q(\cdot\mid\vx,t)}\left[\nabla_\vx \ln q(\vy\mid\vx,t)\right]
	= \nabla_\vx \int d\vy\, q(\vy \mid \vx, t) = 0,
\end{equation}
its posterior expectation is zero point-wise in $\vx$, so only the first, model-independent term survives and the gap carries no dependence on the model:
\begin{align}
	\SDSM - \Sideal
	&= \frac{1}{2}\int_0^T dt\, \gamma(t)\; \E_{\vx \sim P_t}
	\E_{\vy \sim q(\vy\mid\vx,t)} \sqnorm{ \nabla_\vx \ln q(\vy \mid \vx, t) } \nonumber \\
	&= \frac{1}{2}\int_0^T dt\, \gamma(t)\; \E_{\vx \sim P_t}
	\left[ \mathrm{tr}\, \Cov_{\vy \mid \vx}\!\left[ \nabla_\vx \ln q(\vx, t \mid \vy) \right] \right],
	\label{eq:gap_variance}
\end{align}
where the second equality uses Bayes' rule, $\nabla_\vx \ln q(\vy\mid\vx,t) = \nabla_\vx \ln q(\vx,t\mid\vy) - \nabla_\vx \ln P(\vx,t)$, by which the posterior score is the centered conditional score.

\begin{theorem}[The loss floor is a Fisher information] \label{thm:floor_fisher}
	Let $q(\vx,t\mid\vy)$ be a corruption kernel satisfying the following regularity
	conditions: $q(\vx,t\mid\vy) > 0$ on the support of interest; $\vx \mapsto \ln q(\vx,t\mid\vy)$
	is differentiable; and the gradient $\nabla_{\vx} q(\vx,t \mid \vy)$ is uniformly dominated by an integrable function of $\vy$ in a sufficiently small neighbourhood of each $\vx$, so that differentiation under the integral sign is permitted. Let $P_t$ be the induced marginal and let
	\begin{equation} \label{eq:cond_endpoint_family}
		p(\vy \mid \vx,t) \;=\; \frac{P_{\rm data}(\vy)\,q(\vx,t\mid\vy)}{P_t(\vx)}
	\end{equation}
	be the \emph{conditional endpoint family}: the family of laws over clean
	endpoints $\vy$ indexed by the latent coordinate $\vx$, which plays the role of
	the parameter. Let
	\begin{equation} \label{eq:posterior_fisher}
		g_{ij}(\vx,t) \;:=\; \E_{\vy \mid \vx}\big[\partial_i \ln p(\vy\mid\vx,t)\,\partial_j \ln p(\vy\mid\vx,t)\big]
	\end{equation}
	be its Fisher--Rao metric. Then
	\begin{equation} \label{eq:thm_floor}
		\SDSM - \Sideal
		\;=\; \frac{1}{2}\int_0^T dt\,\gamma(t)\;\E_{\vx\sim P_t}\big[\operatorname{tr} g(\vx,t)\big].
	\end{equation}
\end{theorem}

\begin{proof}
	Since $p(\vy\mid\vx,t) = P_{\rm data}(\vy)q(\vx,t\mid\vy)/P_t(\vx)$ and $P_{\rm data}(\vy)$ does not depend on $\vx$, differentiating the logarithm with respect to $\vx$ gives
	\begin{equation} \label{eq:posterior_score}
		\nabla_\vx \ln p(\vy\mid\vx,t) \;=\; \nabla_\vx \ln q(\vx,t\mid\vy) \;-\; \nabla_\vx \ln P_t(\vx).
	\end{equation}
	Taking $\E_{\vy\mid\vx}$ of Eq.~\eqref{eq:posterior_score} and using $\E_{\vy\mid\vx}[\nabla_\vx \ln p]=\nabla_\vx\!\int d\vy\, p(\vy\mid\vx,t) = 0$ recovers the unbiasedness identity of denoising score matching: the conditional score is an unbiased estimator of the marginal score. The posterior score is therefore centered under the posterior, and its second moment equals its covariance,
	\[
	\E_{\vy\mid\vx}\left[ \nabla_\vx \ln p(\vy\mid\vx,t)\,\nabla_\vx \ln p(\vy\mid\vx,t)^\top \right]
	= g(\vx,t),
	\]
	which is simultaneously the Fisher--Rao metric \eqref{eq:posterior_fisher} and the integrand of Eq.~\eqref{eq:gap_variance}. Taking the trace and inserting into Eq.~\eqref{eq:gap_variance} gives Eq.~\eqref{eq:thm_floor}.
\end{proof}

Beyond the stated regularity conditions, the theorem assumes nothing about the corruption kernel, such as Gaussianity, affine structure, or exponential-family form, and its content is a conditional-variance identity. The ideal objective $\Sideal$ compares the model against $\nabla \ln P_t$, which by Eq.~\eqref{eq:posterior_score} is the posterior \emph{mean} of the conditional score. The implementable objective $\SDSM$ cannot evaluate that mean and substitutes a single draw $\vy \sim p(\cdot\mid\vx,t)$, which is unbiased but noisy. The point-wise decomposition now reads
\begin{equation} \label{eq:mse_split}
\underbrace{\E\big\|\mathbf{s}_\theta - \mathbf{s}_{\rm cond}\big\|^2}_{\text{denoising objective}}
\;=\; \underbrace{\big\|\mathbf{s}_\theta - \nabla\!\ln P_t\big\|^2}_{\text{estimation error}}
\;+\; \underbrace{\Var\big(\mathbf{s}_{\rm cond}\mid \vx\big)}_{\text{irreducible conditional variance}} ,
\end{equation}
in which the second term on the right-hand side is the Fisher information of the latent immersion $\iota_t:\vx\mapsto p(\cdot\mid\vx,t)$ of Eq.~\eqref{eq:endpoint_immersion}. This is the same computation as that of a bias--variance decomposition, though the first term is a squared estimation error rather than a bias, since $\mathbf{s}_\theta$ is deterministic given $\vx$. Since the conditional variance of a regression target does not involve the regressor, the floor cannot depend on $\theta$. In the language of Sec.~\ref{subsec:regular_family}: \emph{the latent Fisher metric is the covariance of the random target introduced by denoising score matching}, and the loss floor is its integrated trace along the generative flow.

For the affine Gaussian kernel, the conditional score is $(\alpha_t\vy - \vx)/\sigma_t^2$, and Eq.~\eqref{eq:posterior_score} yields $\nabla_\vx \ln p(\vy\mid\vx,t) = (\alpha_t/\sigma_t^2)(\vy - \E[\vy\mid\vx])$, so the metric reduces to the closed form Eq.~\eqref{eq:g_chi}. 

\section{The Loss Floor as Information Flow}
\label{sec:infoflow}

Theorem~\ref{thm:floor_fisher} identified the irreducible part of the denoising objective with a Fisher--Rao metric of the endpoint posterior, and for affine kernels, the explicit expression is given by Eq.~\eqref{eq:g_chi}. The remaining question is what the metric measures when accumulated along the diffusion trajectory. Remarkably, the time integral has a simple information-theoretic form: it is the amount of information about the clean endpoint that is lost under corruption.

This interpretation also clarifies several properties of the loss floor that otherwise appear unrelated. The familiar schedule invariance of the diffusion objective becomes a consequence of reparametrization invariance. Moreover, the divergence near the data end is controlled by the information dimension of the data, and the decomposition admits a natural thermodynamic interpretation. Proofs are collected in Appendix~\ref{app:infoflow}, with numerical details in Appendix~\ref{app:numerics}.

\subsection{The mutual-information identity}
\label{subsec:mi_identity}

The Fisher--Rao representation of Theorem~\ref{thm:floor_fisher} is local in the corruption parameter. To understand the accumulated floor, we first ask what its instantaneous integrand measures. Although the explicit formulas below are often encountered in the special case of affine Gaussian corruption, the first step requires only the corruption kernel itself. Let
\[
J(p) := \E_p\|\nabla \ln p\|^2
\]
denote the Fisher information of a density with respect to translations.

The answer to the above question is particularly simple. Averaged over the noisy marginal, the trace of the posterior Fisher metric is exactly the Fisher information lost when the endpoint-conditioned corruption kernels are mixed over the data distribution.

\begin{lemma}[Fisher-information form of the floor]
	\label{lem:fisher_gap}
	For any corruption kernel,
	\begin{equation}
		\E_{P_t}\big[\operatorname{tr} g(\cdot,t)\big]
		=
		\E_{\vy}\big[J(q_t(\cdot\mid\vy))\big] - J(P_t).
		\label{eq:fisher_gap}
	\end{equation}
\end{lemma}

Thus the instantaneous floor is the gap between the Fisher information available when the clean endpoint $\vy$ is known and that remaining after $\vy$ has been marginalized out. Its non-negativity is the corresponding monotonicity of Fisher information under mixing \cite{CoverThomas2006}.

To turn this local identity into a statement about the full loss floor, we now let the corruption be generated continuously by a diffusion. Suppose that
\[
d\vx_t = f(\vx_t,t)\,dt + \sqrt{D(t)}\,d\vW_t
\]
starts from the data distribution, where $f$ and $D$ are independent of $\vy$, and $t$ increases in the corruption direction. We write $D(t)$ rather than $\gamma(t)$ to distinguish the corruption process considered here from the bridge dynamics of Sec.~\ref{sec:bridge}. Along such a diffusion, the same Fisher-information gap determines the rate at which the noisy variable forgets the clean endpoint:
\begin{lemma}[de Bruijn identity along a corruption diffusion]
	\label{lem:debruijn}
	Under the assumptions above,
	\begin{equation}
		\frac{d}{dt} I(\vy;\vx_t)
		=
		-\frac{D(t)}{2}
		\Big(
		\E_\vy J(q_t(\cdot\mid\vy)) - J(P_t)
		\Big)
		=
		-\frac{D(t)}{2}\,
		\E_{P_t}\big[\operatorname{tr} g(\cdot,t)\big]
		\leq 0.
		\label{eq:debruijn}
	\end{equation}
\end{lemma}

Eq~\eqref{eq:debruijn} gives the information-theoretic meaning of the metric trace. It is, up to the local diffusion scale, the instantaneous rate at which information about the clean endpoint is erased by the corruption process. Integrating this identity converts the local Fisher geometry of the denoising loss into a global information-flow law. The de Bruijn identity relates entropy production under Gaussian diffusion to Fisher information \cite{CoverThomas2006, GuoShamaiVerdu2005}, and we have shown that the same structure appears as the difference between the conditional and marginal Fisher informations and, through Theorem~\ref{thm:floor_fisher}, as the trace of the endpoint-posterior metric.

\begin{theorem}[General factorization of the floor] \label{thm:general_floor}
Under the hypotheses of Lemma~\ref{lem:debruijn}, for an arbitrary loss weighting $w(t)$,
\begin{equation} \label{eq:general_factorization}
\SDSM - \Sideal
\;=\; \frac12\int w(t)\,\E_{P_t}\big[\operatorname{tr} g\big]\,dt
\;=\; -\int \underbrace{\frac{w(t)}{D(t)}}_{\text{schedule}}\;
      \underbrace{\frac{dI(\vy;\vx_t)}{dt}}_{\text{data}}\,dt .
\end{equation}
In particular, if the loss weighting equals the diffusion coefficient, $w = D$, the floor is simply:
$\SDSM-\Sideal = I(\vy;\vx_{t_0}) - I(\vy;\vx_{t_1})$.
\end{theorem}

In making this statement, notice that once again, we made no assumptions about Gaussian, affine interpolant or signal-to-noise ratio. The data enters only through the information flow $-dI/dt$, and the schedule only through the ratio $w/D$. We have verified Eq.~\eqref{eq:debruijn} numerically for a nonlinear corruption drift by solving the Fokker--Planck equation directly (Appendix~\ref{app:numerics}).

\subsection{Schedule and weighting}
\label{subsec:schedule}

The word \emph{schedule} is used in the diffusion literature for several different objects, and since Theorem~\ref{thm:general_floor} assigns these objects distinct roles, we separate them before evaluating the floor. The primary one is the \emph{corruption schedule} $(\alpha_t,\sigma_t)$, which specifies how the data are gradually destroyed by noise; equivalently, one may specify the drift--diffusion pair $(\vb,\gamma)$ of the forward process or the SNR profile $\lambda_t$. This is the only one of the four that shapes the integrand itself, because it determines the curve that the corruption traces in the information manifold. A second object, the \emph{loss weighting} $w(t)$, multiplies the contribution of each time to the objective. A third, the \emph{training-time density}, is the distribution from which noise levels are drawn during optimization. Neither of these two enters the integrand and together they determine only how the integral is weighted and how it is sampled. The fourth, the \emph{sampling-time discretization}, is the step placement of a numerical solver integrating the generative dynamics and plays no role in the training objective at all, and we defer it to Sec.~\ref{subsec:third_order}. With these distinctions in place, we return to the bridge construction of Sec.~\ref{sec:bridge}, which singles out a distinguished combination of schedule and weighting:

\begin{lemma}[Bridge weighting is the SNR measure] \label{lem:schedule}
Write $\lambda_t := \alpha_t^2/\sigma_t^2$ for the signal-to-noise ratio. Among the marginal-preserving diffusions that realize a given affine interpolant, the bridge construction of Sec.~\ref{sec:bridge} singles out the memoryless one, that is, the unique state-independent coefficient for which the rescaled process $\vx_t/\alpha_t$ has independent increments, namely
\begin{equation} \label{eq:bridge_coefficient}
	\gamma(t) \;=\; 2\,\sigma_t^2\,\frac{d}{dt}\ln\frac{\alpha_t}{\sigma_t},
\end{equation}
and for this coefficient, for any affine interpolant,
\begin{equation} \label{eq:schedule_identity}
\gamma(t)\,\frac{\alpha_t^2}{\sigma_t^4} \;=\; \frac{d\lambda_t}{dt}.
\end{equation}
\end{lemma}

\noindent The statement can be verified by a direct computation for each interpolant family and is given in Appendix~\ref{app:numerics}. The lemma is also of independent interest beyond its use below, since it implies \emph{the memoryless schedule is precisely the one whose loss weighting is the SNR measure}. In other words, the bridge member of the marginal-preserving family is not one convenient choice among many but the canonical object singled out by the objective itself.

With the four notions separated, we can now factor the floor into a data contribution and a schedule contribution. Suppose the training loss uses a weight $w(t)$ in place of the bridge weight $\gamma$. Changing variables from $t$ to $\lambda$ with Lemma~\ref{lem:schedule} then gives
\begin{equation} \label{eq:floor_factorized}
\SDSM-\Sideal
\;=\; \frac12 \int \omega(\lambda)\,\mathrm{MMSE}(\lambda)\,d\lambda
\;=\; \frac12 \int \omega(\lambda)\,S(\lambda)\,d\log\lambda,
\qquad \omega := \frac{w}{\gamma},
\end{equation}
where $\mathrm{MMSE}(\lambda) := \E\,\|\vy - \E[\vy\mid \vx_\lambda]\|^2$ with $\vx_\lambda = \sqrt{\lambda}\,\vy + \vz$, $\vz\sim\mathcal N(0,\mathbf I)$, and we have defined the \emph{information spectrum} of the data:
\begin{equation} \label{eq:info_spectrum}
S(\lambda) \;:=\; \lambda\,\mathrm{MMSE}(\lambda) \;=\; 2\,\frac{dI(\vy;\vx_\lambda)}{d\log\lambda}.
\end{equation}
It is easy to see that the data dependency enters only through $S$, a scalar function of log-SNR determined by the distribution alone, while the schedule dependence enters only through $\omega$ and the endpoints, with the bridge weighting corresponding to $\omega \equiv 1$.

In this case, the spectrum has a transparent structure, shown in Fig.~\ref{fig:spectrum}. Each structural scale of the data contributes a bump, so that for a mixture with modes separated by $m$ and within-mode scale $s$, mode identity is resolved near $\lambda \sim m^{-2}$ and the within-mode directions near $\lambda\sim s^{-2}$, while at large $\lambda$ the spectrum saturates at the information dimension, $S \to d(\vy)$, which is Eq.~\eqref{eq:info_dimension} below. 

This structure suggests a criterion for allocating resolution. If one asks that each unit of schedule carry equal information, a natural choice, since the excess $\Sideal$ is what the model must learn while $S$ sets the scale of the noise it must learn through, then the allocation satisfies
\begin{equation} \label{eq:adaptive_schedule}
\frac{dt}{d\log\lambda} \;\propto\; S(\lambda).
\end{equation}
For featureless data the spectrum $S$ is flat, and the criterion reduces to uniform allocation in $\log\lambda$. This is qualitatively close to the commonly used cosine-like schedules over finite, clipped SNR ranges, though not identical to them, and related derivations of cosine-like schedules from the schedule-dependent part of the geometry appear in \cite{CosineFisherRao2025, Kong2026Schedule}. Real data, however, are not featureless. The spectrum carries a bump at each structural scale of the data, so a uniform-in-$\log\lambda$ allocation under-samples precisely the noise levels at which the data are most informative. By the discussion of Sec.~\ref{subsec:highsnr}, these are the neighbourhoods of the symmetry-breaking transitions. The criterion also has a direct information-theoretic meaning. Since $-dI/dt = dH(\vy\mid\vx_t)/dt$, Eq.~\eqref{eq:adaptive_schedule} coincides with the criterion of entropic time schedulers \cite{Stancevic2025Entropic}, which we have thus recovered from the loss decomposition rather than from an entropy argument. Equation~\eqref{eq:general_factorization} shows that this criterion is not merely plausible but forced by the structure of the objective, because the floor factorizes into a data flow and a schedule ratio, the total is pinned by the endpoints, and equal-information allocation is precisely the choice that makes the schedule factor uniform in the data's own coordinate. Finally, training-time noise allocation has itself been studied \cite{Ambrogioni2026Atoms}. A narrower variant that we flag as an open question, and do not pursue here, is to use $S(\lambda)$ as an \emph{online} importance density for Monte Carlo sampling of training noise levels, estimated from the denoiser's posterior variance as training proceeds.

\begin{figure}[t]
\centering
\includegraphics[width=\textwidth]{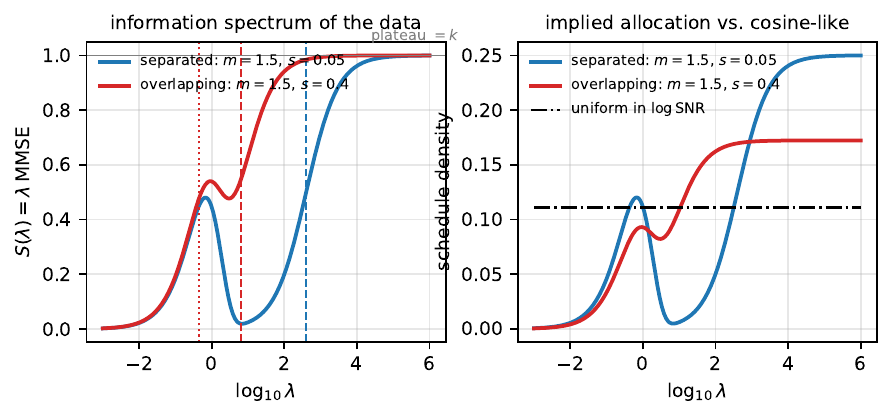}
\caption{Data/schedule separation of the loss floor. \textbf{Left:} the information spectrum $S(\lambda)=\lambda\,\mathrm{MMSE}(\lambda)$ for two mixtures. Dotted lines mark the mode-separation scale $\lambda=m^{-2}$, dashed lines the within-mode scale $\lambda=s^{-2}$; the plateau at large $\lambda$ is the intrinsic dimension $k=1$. \textbf{Right:} the schedule density implied by Eq.~\eqref{eq:adaptive_schedule}, compared with a uniform-in-$\log$SNR allocation. Uniform allocation is correct only for featureless data.}
\label{fig:spectrum}
\end{figure}

\subsection{Gaussian corruption and reparametrization invariance}
\label{subsec:gaussian}

The affine case is the specialization in which the information flow has a closed form. For affine kernels, Eq.~\eqref{eq:g_chi} makes $\operatorname{tr}g$ proportional to $\mathrm{MMSE}(\lambda)$.

\begin{theorem}[Loss floor $=$ information gained] \label{thm:floor_information}
With the bridge schedule of Lemma~\ref{lem:schedule},
\begin{equation} \label{eq:floor_information}
\SDSM - \Sideal
\;=\; \frac{1}{2}\int \mathrm{MMSE}(\lambda)\,d\lambda
\;=\; I(\vy; \vx_{\lambda_1}) - I(\vy;\vx_{\lambda_0})
\;=\; h(\vx_{\lambda_1}) - h(\vx_{\lambda_0}),
\end{equation}
where $I$ is mutual information and $h$ differential entropy.
\end{theorem}

\noindent The proof combines Theorem~\ref{thm:floor_fisher} with Lemma~\ref{lem:schedule} and the I--MMSE relation of Guo, Shamai and Verd\'u \cite{GuoShamaiVerdu2005} (Appendix~\ref{app:infoflow}). Figure~\ref{fig:floor} verifies Eq.~\eqref{eq:floor_information} numerically on a Gaussian mixture (Appendix~\ref{app:numerics}).

\begin{figure}[t]
\centering
\includegraphics[width=\textwidth]{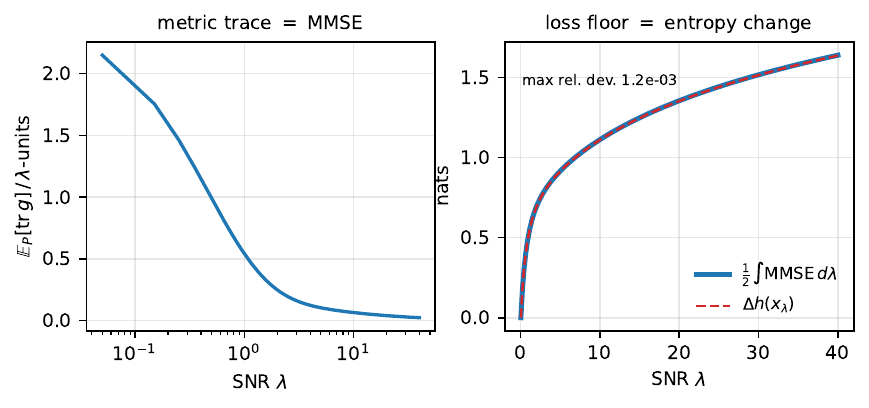}
\caption{The loss floor is an entropy. \textbf{Left:} the trace of the latent Fisher--Rao metric, equivalently the MMSE of endpoint estimation, along the flow. \textbf{Right:} the accumulated floor $\frac12\int\mathrm{MMSE}\,d\lambda$ (solid) against the differential-entropy change $\Delta h(\vx_\lambda)$ (dashed), verifying Theorem~\ref{thm:floor_information}. Two-mode Gaussian mixture, entropies by quadrature.}
\label{fig:floor}
\end{figure}

Because the right-hand side of Eq.~\eqref{eq:floor_information} is a difference of endpoint values, the loss floor depends on the noise schedule only through the endpoints of the SNR range, not on the path between them. This is the invariance theorem of variational diffusion models \cite{kingma2021variational}, obtained here as a integral rather than by direct computation. The geometric interpretation of this is that, a schedule is a \emph{parametrization} of a fixed curve in the information manifold, and the loss floor is a parametrization-invariant line integral along that curve, pinned by its endpoints.

Invariance is thus not an algebraic accident of the ELBO but the statement that a geometric quantity does not depend on how one traverses the curve, and this also delimits what schedule design can achieve. A schedule cannot change the floor, it can only redistribute the estimation error along the path. In that sense, the recent Fisher-geometric schedule derivations \cite{CosineFisherRao2025, Kong2026Schedule} and the present account are complementary.

\subsection{High-SNR asymptotics and information dimension}
\label{subsec:highsnr}

For continuous data, the mutual information $I(\vy;\vx_\lambda)$ diverges as $\lambda\to\infty$, because perfect observation of a continuous variable carries unbounded information. Theorem~\ref{thm:floor_information} then implies that the floor is infinite in the continuum limit and that it is finite in practice only because implementations truncate the SNR range. The growth of the metric trace in Eq.~\eqref{eq:g_chi} as $\sigma_t\to 0$ is therefore not a defect of the objective. It is the divergence of a mutual information, and the familiar $t_{\min}$ cut-off is more accurately understood as an information cut-off. The rate at which the floor diverges is itself informative. We first give the Gaussian computation, which is elementary, and then state the general asymptotic, which rests on a deeper result from information theory.

For \emph{Gaussian} $\vy$ with covariance eigenvalues $s_i^2$, the posterior is Gaussian as well, with variance $s_i^2/(1+\lambda s_i^2)$ in each eigendirection. Summing these variances gives
\begin{equation} \label{eq:mmse_gauss}
	\mathrm{MMSE}(\lambda) \;=\; \sum_i \frac{s_i^2}{1+\lambda s_i^2}
	\;\xrightarrow[\lambda\to\infty]{}\; \frac{k}{\lambda},
	\qquad
	\SDSM - \Sideal\Big|_{\lambda\le\Lambda} \sim \frac{k}{2}\log\Lambda,
\end{equation}
where $k$ is the number of non-zero eigenvalues. We stress that this computation uses Gaussianity of the prior in an essential way and does not follow from the support and covariance of the data alone.

The logarithmic rate itself is nevertheless general. What changes in the general case is the coefficient, which is an information-theoretic quantity rather than a linear-algebraic one. The high-SNR behaviour of $I(\vy;\vx_\lambda)$ in a Gaussian channel is governed by the R\'enyi information dimension of the input \cite{Renyi1959}, and correspondingly $\lambda\,\mathrm{MMSE}(\lambda)$ tends to the MMSE dimension of $\vy$ \cite{WuVerdu2010, WuVerdu2011}. Writing $d(\vy)$ for this dimension, the floor behaves as
\begin{equation} \label{eq:info_dimension}
	\SDSM - \Sideal\Big|_{\lambda\le\Lambda}
	\;\sim\; \frac{d(\vy)}{2}\,\log\Lambda .
\end{equation}
This single formula covers several regimes. For a distribution that is regular on a $k$-dimensional manifold, $d(\vy)=k$ and Eq.~\eqref{eq:info_dimension} reduces to Eq.~\eqref{eq:mmse_gauss}. For a mixture of discrete and continuous components, $d(\vy)$ equals the weight of the continuous part and need not be an integer. For singular distributions, the information dimension may fail to exist altogether. The divergence rate of the loss floor is therefore a property of the information dimension of the data, and it can be estimated from the denoiser's posterior variance alone, without access to a likelihood or to the Jacobian of the network. Figure~\ref{fig:dim} confirms Eq.~\eqref{eq:mmse_gauss} numerically, and the fitted slopes match $k/2$ to three decimals. We note that estimating intrinsic dimension from diffusion models is an established topic \cite{Pope2021ID, Stanczuk2024ID, Kamkari2024LID}. Our point here is not to propose a competitive estimator but to explain why the loss floor knows the dimension at all.

\begin{figure}[t]
	\centering
	\begin{minipage}[t]{0.46\textwidth}
		\centering
		\includegraphics[width=\textwidth]{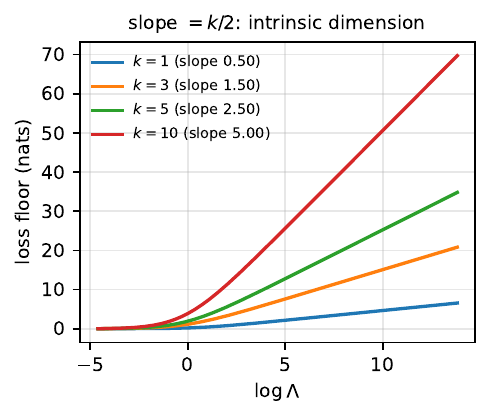}
		\caption{The floor diverges logarithmically at the data end, with slope $k/2$ for Gaussian data of rank $k$ in $\R^{10}$.}
		\label{fig:dim}
	\end{minipage}\hfill
	\begin{minipage}[t]{0.50\textwidth}
		\centering
		\includegraphics[width=\textwidth]{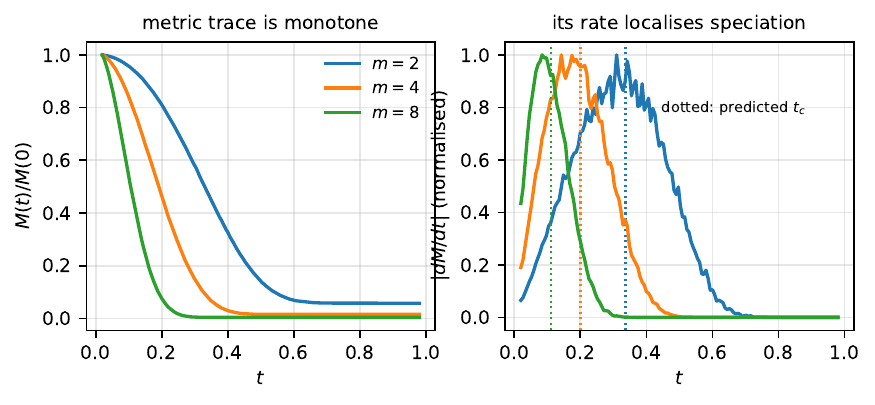}
		\caption{The metric trace is monotone (left), while its \emph{rate} localizes the symmetry-breaking transition (right), with dotted lines marking the predicted bimodality threshold.}
		\label{fig:speciation}
	\end{minipage}
\end{figure}

These asymptotics sit close to the symmetry-breaking transitions studied in \cite{Raya2023SpontaneousSB, Ambrogioni2023TheST, Biroli2024DynamicalRO}, and it is tempting to expect the floor itself to diverge at those transitions. It does not. To see why, write $M(t) := \E_{P_t}\operatorname{tr} \Cov[\vy\mid\vx,t]$ for the intrinsic part of the floor. This quantity decreases monotonically from its value at maximal ambiguity to the within-mode variance, so it remains finite throughout the transition. What localizes the transition is instead its rate $|dM/dt|$, which sharpens as the modes separate, as shown in Fig.~\ref{fig:speciation}. The geometric signature of speciation is thus the rate of change of the latent metric rather than the metric itself, a statement close in content to the entropy-rate criterion of \cite{EntropicSpeciation2026}.

\subsection{A thermodynamic interpretation}
\label{subsec:thermo}

We close this section with a thermodynamic interpretation of the decomposition. This interpretation is an analogy rather than an identity, and we will be careful about where it holds and where it breaks.

The starting point is the notion of entropy production in stochastic thermodynamics. There, entropy production is a path-space relative entropy between a process and its time reversal, $\sigma_\tau = \lim_{dt\to 0}D_{KL}(P\|P^\dagger)/dt$, where $P^\dagger$ denotes the time-reversed path measure \cite{Seifert2012, Ito2024GeometricThermo}. A related construction equips the space of path measures with a Fisher metric by interpolating between two velocity fields, and the induced line element is $\frac{dt}{2\mu T}\int\|\nu'-\nu\|^2 P$ \cite{Ito2024GeometricThermo}.

Our excess action Eq.~\eqref{eq:excess_action_integral} has exactly this form, with the diffusion coefficient $\gamma$ playing the role of $2\mu T$. It can therefore be read as the dissipation incurred by running the model's velocity field in place of the optimal one. Theorem~\ref{thm:floor_information} complements this picture, since it identifies the floor with $\Delta h$, an entropy change of the state along the reference corruption channel.

However, we need to be careful before combining these two observations. The quantity $\Delta h$ is a change of differential entropy of the state. Entropy production in stochastic thermodynamics is instead a path-space quantity with separate system and environment contributions. The two coincide only under the Gaussian-channel reading in the SNR parametrization. With this caveat in place, the decomposition resembles a dissipation-plus-entropy splitting,
\begin{equation} \label{eq:thermo_decomposition}
	\underbrace{\SDSM}_{\text{denoising loss}}
	\;=\;\underbrace{\Sideal}_{\text{dissipation}}
	\;+\;\underbrace{\Delta h}_{\text{entropy production}} .
\end{equation}

That the score-matching objective is itself an entropy-production functional has been established independently and in more detail by \cite{ScoreMatchingThermo2026}, whose time-asymmetry entropy production is proportional to the score-matching loss. Our contribution is complementary, since we identify the \emph{irreducible} term with a Fisher--Rao metric and with $\Delta h$. We also caution that Eq.~\eqref{eq:thermo_decomposition} is \emph{not} the excess/housekeeping decomposition of \cite{Ito2024GeometricThermo}. That decomposition splits entropy production by a geometric criterion, separating the Wasserstein speed of the density from the remainder, whereas ours splits the loss by dependence on the model. The two coincide in neither definition nor value.

\section{Consequences for Training and Sampling}
\label{sec:consequences}

The decomposition of Sec.~\ref{sec:infoflow} affects the two halves of the pipeline differently, because training and sampling probe different parts of the geometry. The training objective involves only the second cumulant of the denoising posterior, while the discretization error of the sampler also involves the third.

\subsection{Why raw training losses are not comparable}
\label{subsec:losscompare}

For the bridge weighting, the decomposition reads $\SDSM = \Sideal + \Delta h$, and the floor depends on the schedule through its endpoints and on the data through its entropy. Raw training losses are therefore not directly comparable across runs that use different SNR ranges, schedules, or loss weightings. Under fixed training conditions the floor is common to all runs and the conventional loss remains a valid comparator, but cross-configuration comparisons can be genuinely misleading. Figure~\ref{fig:consequences} (left, centre) gives an example on the analytic mixture of Appendix~\ref{app:numerics}. With two SNR ranges and two models of fixed known quality, the better model evaluated on the wider range reports a larger raw loss ($1.5455$) than the worse model on the narrower range ($1.1267$), purely because the wider range integrates more of the floor. Subtracting the floors ($1.0414$ and $1.5143$) gives excesses that order the models correctly on both ranges. We therefore recommend reporting the floor-subtracted excess alongside the conventional loss whenever models are compared across schedules or SNR ranges. The correction is cheap, since Theorem~\ref{thm:floor_fisher} expresses the floor as an expectation of the posterior covariance, which the denoising loop already has the ingredients to estimate.

\begin{figure}[t]
	\centering
	\includegraphics[width=\textwidth]{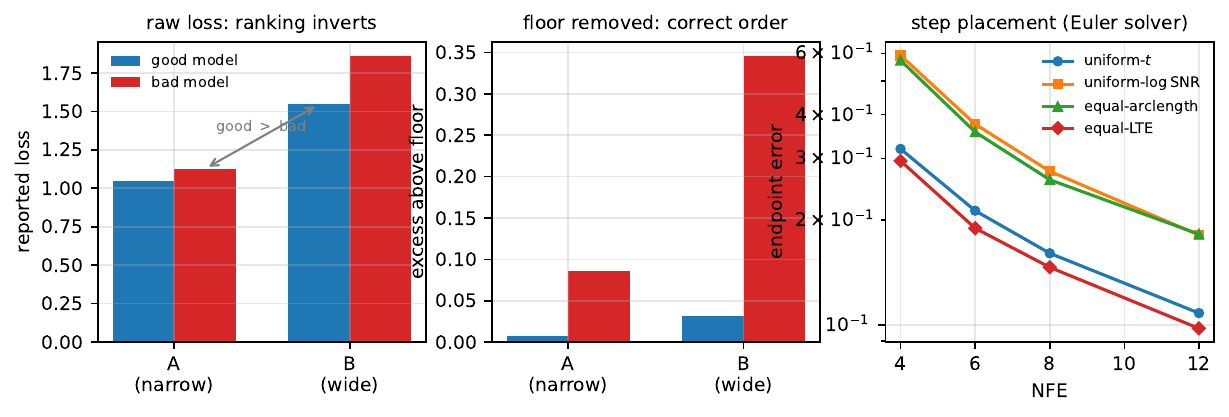}
	\caption{Consequences of the decomposition. \textbf{Left:} raw training loss for two models of known quality on two SNR ranges; the ranking inverts. \textbf{Centre:} after subtracting the floor, the excess orders the models correctly on both ranges. \textbf{Right:} step placement for a first-order solver on the analytic mixture, mean endpoint error against the exactly integrated probability-flow ODE.}
	\label{fig:consequences}
\end{figure}

\subsection{What schedule design can and cannot change}
\label{subsec:schedule_limits}

The decomposition also delimits schedule design. Because the total floor is pinned by the endpoints of the SNR range, a schedule can only redistribute where the estimation error is incurred, and the equal-information allocation discussed after Eq.~\eqref{eq:adaptive_schedule} recovers the criterion already used by entropic time schedulers \cite{Stancevic2025Entropic}. The one variant that appears unclaimed is the training-time analogue, namely sampling $t$ from a density proportional to $S(\lambda) = \operatorname{tr}\Cov[\vy\mid\vx_t]$ estimated online from the denoiser, which we record as an open question rather than a proposal.

\subsection{Second- versus third-order conditional statistics}
\label{subsec:third_order}

This subsection is exploratory and reports diagnostic observations on an analytic example rather than sampler design principles. Higher cumulants of the denoising posterior are absent from the training objective, because denoising score matching draws each noise level independently and never discretizes a trajectory. They appear as soon as one integrates the sampler. Writing the probability-flow trajectory in the mean-parameter coordinate $\eta = \nabla\Lambda = \E[\vy\mid\vx_t]$ with tangent $u = d\theta/d\lambda$,
\begin{equation} \label{eq:cumulant_expansion}
	\frac{d\eta}{d\lambda} \;=\; g\,u,
	\qquad
	\frac{d^2\eta}{d\lambda^2} \;=\; T[u,u] \;+\; g\,\frac{du}{d\lambda},
	\qquad T := \nabla^3\Lambda ,
\end{equation}
where $T$ is the third cumulant of the conditional endpoint law. The leading truncation term of a first-order step is therefore $d^2\eta/d\lambda^2$, which depends on the third cumulant through $T[u,u]$ in addition to the second-order term $g\,du/d\lambda$. Measured along the exactly integrated trajectory on the analytic mixture (Appendix~\ref{app:numerics}), the contribution of $T$ is strongly localized in time and peaks where the modes separate, with a median share of $0.47$ of $\|d^2\eta/d\lambda^2\|$ across trajectories. Accordingly, Figure~\ref{fig:consequences} (right) shows that equalizing the true leading error $\|d^2\eta/d\lambda^2\|\Delta\lambda^2$ outperforms equalizing Fisher--Rao arclength at every step count tested. We report this as a diagnostic rather than a sampler proposal, because the comparison uses an explicit Euler solver while standard baselines such as DDIM \cite{Song2020DDIM} integrate the linear part exactly, and the decisive experiment with an exponential integrator at matched order remains to be carried out.

\section{Discussion}
\label{sec:discussion}

The identity behind every result in this paper is elementary. Denoising score matching regresses onto a random target whose posterior mean is the marginal score, and the fluctuation of that target is the score of the conditional endpoint family. Its second moment is a Fisher information by definition, which is why the training loss contains a term no model can reduce and why that term is the metric of the latent manifold. Once the corruption is a diffusion, the same quantity measures the rate at which the noisy state forgets the data, so the floor accumulates a mutual information. The geometry, the information theory, and the objective are three descriptions of one number.

\subsection{Scope and relation to prior work}
\label{subsec:scope}

The primary results are the conditional-variance identity of Theorem~\ref{thm:floor_fisher} and its factorization into an information flow and a schedule weight in Theorem~\ref{thm:general_floor}. The ingredients of Sec.~\ref{sec:fisher} are classical, namely the second variation of a log-partition functional, the theory of sufficient reductions, and the induced geometry of curved exponential families \cite{Chentsov1982, Efron1975, Amari1982}, and we claim novelty only for their assembly into a derivation of the latent metric from the bridge variational principle. Several nearby lines of work address optimization problems that are easily conflated with ours, including sampler discretization \cite{Sabour2024AYS, Williams2024ScoreOptimal, Stancevic2025Entropic}, the training-time noise density \cite{Ambrogioni2026Atoms}, loss weighting \cite{kingma2021variational}, and time reparametrization \cite{CosineFisherRao2025, Kong2026Schedule}. Our contribution to this thread is not a new member of any of these families but the factorization Eq.~\eqref{eq:floor_factorized}, which delimits what any of them can affect. Two further boundaries are worth marking. The reading of the score-matching objective as an entropy-production functional is developed independently and in more detail by \cite{ScoreMatchingThermo2026}, while our contribution is the identification of the irreducible term with a Fisher--Rao metric and with an entropy change. Estimating intrinsic dimension from a trained diffusion model is established \cite{Pope2021ID, Stanczuk2024ID, Kamkari2024LID}, and Eq.~\eqref{eq:info_dimension} explains why the floor carries that information rather than proposing a new estimator.

\subsection{The discrete case of masked diffusion}
\label{subsec:discrete}

The decomposition also has a definite consequence beyond the continuous setting, which we believe is among the most useful observations of this work. The masked-diffusion objective is a cross-entropy, and a cross-entropy is the Bregman divergence generated by the negative entropy. The same algebra that produced the conditional-variance split Eq.~\eqref{eq:mse_split} applies to any such loss. The expected divergence to a random target separates into a model-dependent term evaluated at the posterior mean of the target and an irreducible remainder known as the Bregman information of the target \cite{Savage1971, Banerjee2005Clustering, Banerjee2005Optimality}. For the cross-entropy this remainder is a conditional entropy. Writing $m_t$ for the masking probability, decreasing from $1$ to $0$, and $U$ for the set of unmasked tokens, a still-masked token is revealed over an infinitesimal step with probability $-\dot m_t/m_t$ at a cost $H(y_i \mid y_U)$. With a perfect model the residual loss is therefore
\begin{equation} \label{eq:discrete_floor}
	\mathrm{floor} \;=\; \int_0^1 \frac{dm}{m}\; \E_U\Big[ \sum_{i \notin U} H\big(y_i \mid y_U\big) \Big],
\end{equation}
where each token is unmasked independently with probability $1-m$. Two properties follow immediately. Substituting $u = m_t$ in the time integral removes the schedule, so the floor depends only on the endpoint masking rates and not on the shape of $m_t$. Evaluating Eq.~\eqref{eq:discrete_floor} then gives the entropy $H(y_1,\ldots,y_n)$ of the data. We have verified both numerically for correlated tokens (Appendix~\ref{app:numerics}).

This is the discrete counterpart of the continuous statement. There the floor is an accumulated mutual information fixed by the endpoints of the SNR range, and here it is the data entropy fixed by the endpoint masking rates. In both cases the schedule redistributes where the cost is incurred without changing the total. This also accounts for the known schedule invariance of the masked-diffusion evidence bound \cite{Sahoo2024MDLM, Shi2024MD4}, because the objective is a Bregman divergence whose irreducible part telescopes along the filtration. We leave open the harder question of a discrete analogue of the cone Eq.~\eqref{eq:cone} and of the latent metric itself, which would require the $\alpha$-geometry of the simplex rather than the argument given here.

\subsection{Limitations}
\label{subsec:limitations}

Several limitations should be kept in view. The closed-form evaluation of the floor as $\Delta I$ assumes affine Gaussian corruption, and Theorem~\ref{thm:general_floor} covers general corruption diffusions only at the price of leaving the information flow implicit. The high-SNR rate is governed by the information dimension, which need not exist for singular data distributions. The replacement of the on-policy weighting by the simulation-free one in Sec.~\ref{subsec:excess} leaves the functional minimizer unchanged, but the projection onto a restricted parametric model does depend on the weighting measure, so our statements concern the value of the objective rather than the parametric projection. Our experiments are analytic or low-dimensional by design and establish identities rather than performance. Finally, the floor is only as estimable as the posterior covariance, which is hardest to estimate at the low-noise end where the floor diverges.

\subsection{Concluding remarks}
\label{subsec:conclusion}

We have shown that the denoising score-matching objective decomposes exactly into a model-dependent estimation error and a model-independent floor, and that the floor is the trace of the Fisher--Rao metric of the conditional endpoint family integrated along the flow. For corruption diffusions the floor factorizes into a data-dependent information flow and a schedule-dependent weight, and for affine Gaussian corruption it evaluates to the mutual information accumulated between data and noisy state. The practical consequence we would emphasize is the least glamorous one. A reported diffusion loss mixes model fit with a data- and schedule-dependent constant, and the two should be separated before the number is used to compare training runs.

\section*{Ethics Statement}

In the preparation of this manuscript we used large language models for language improvement, editing, and literature search and summarization. These tools were not used to produce research ideas, derivations, or results, and every suggested change was reviewed and verified by the authors before inclusion. We take full responsibility for the content of this paper, and any mistakes that remain in the text are entirely ours.

\section*{Acknowledgements}

This work was carried out as an independent research project and received no specific grant from any funding agency in the public, commercial, or not-for-profit sectors. It was conducted outside the authors' official duties, and the views expressed are those of the authors alone, not of their employers.

\bibliographystyle{unsrt}
\bibliography{refs}

\newpage

\appendix
\section*{Appendix}
\section{Variational Derivation of the Endpoint-Tilted Path Measure}
\label{app:schrodinger_bridge_variation}

\renewcommand{\theequation}{\thesection.\arabic{equation}}
\setcounter{equation}{0}

This appendix collects the derivations underlying the variational formulation of the Schr\"odinger bridge in Sec.~\ref{sec:bridge}. We first solve the constrained minimization that yields the factorized optimal path measure, and then derive the recursion relations for the auxiliary fields together with their continuum limit. The remaining subsections expand the tilted kernel at short times, construct the conditional fields that connect the bridge to the score-matching objective, and compute the short-time KL divergence between diffusion kernels that underlies the excess action.

\subsection{Solution of the bridge problem}
Consider the time interval $[0,T]$ discretized into $N$ equal segments of length $\Dt = T/N$, and let $\vx_k$ denote the state at time $t_k = k \Dt$. The joint path measure under the reference process is
\begin{equation}
	\cG(\vx_{0:N}) = P_0(\vx_0) \prod_{k=0}^{N-1} G(\vx_{k+1} \mid \vx_k),
\end{equation}
and under the candidate process,
\begin{equation}
	\cH(\vx_{0:N}) = P_0(\vx_0) H(\vx_{1:N} \mid \vx_0) = P_0(\vx_0) \prod_{k=0}^{N-1} H(\vx_{k+1} \mid \vx_k).
\end{equation}
The Schr\"odinger bridge is formulated directly as the minimization of relative
entropy on path space,
\begin{equation} \label{eq:app_pathkl}
	\begin{split}
		\cHstar &= \argmin_{\cH:\; \cH_0 = P_0,\ \cH_N = P_T}\ D^{\rm path}_{KL}(\cH\|\cG),\\
		D^{\rm path}_{KL}(\cH\|\cG) &= \int d\vx_{0:N}\,\cH(\vx_{0:N})\ln\frac{\cH(\vx_{0:N})}{\cG(\vx_{0:N})},
	\end{split}
\end{equation}
with the endpoint marginals imposed as constraints. No equivalence between a
marginal-propagator divergence and the path-space divergence is required, and we
do not assert one. Equation~\eqref{eq:app_pathkl} is the definition of the problem
in the sense of \cite{Follmer1988, Leonard2014}, and the derivation below is a
variation of that functional alone.

To enforce the marginal constraints at the initial and final times, we introduce Lagrange multiplier functions $\lambda_0(\vx_0)$ and $\lambda_N(\vx_N)$ and form the augmented functional
\begin{align}
	\begin{split}
		\mathcal F[\cH;\lambda_0,\lambda_N] = D^{\text{path}}_{KL}(\cH \| \cG) + \int d\vy\, \lambda_0(\vy) \left[\int d\vx_{1:N}\, \cH(\vy,\vx_{1:N}) - P_0(\vy) \right] \\ + \int d\vy\, \lambda_N(\vy) \left[\int d\vx_{0:N-1}\,\cH(\vx_{0:N-1},\vy) - P_T(\vy)\right].
	\end{split}
\end{align}
Varying $\mathcal{F}$ with respect to $\cH$, treated as a functional derivative in the space of probability measures, and setting the variation to zero gives
\begin{equation}
	\ln \frac{\cHstar}{\cG} + 1 + \lambda_0(\vx_0) + \lambda_N(\vx_N) = 0,
\end{equation}
which rearranges to $\cHstar(\vx_{0:N}) = f(\vx_0)\, \cG(\vx_{0:N})\, g(\vx_N)$, where the endpoint functions are identified as
\begin{subequations}
	\begin{align}
		f(\vx_0) &= e^{-\frac{1}{2} - c - \lambda_0(\vx_0)}, \\
		g(\vx_N) &= e^{-\frac{1}{2} + c - \lambda_N(\vx_N)}.
	\end{align}
\end{subequations}
The constant $c$ reflects a global symmetry, since the transformation $f \rightarrow e^\alpha f$, $g \rightarrow e^{-\alpha} g$ leaves the product $fg$ invariant.

To fix the multipliers explicitly, we enforce the endpoint marginals. The initial marginal of $\cHstar$ is
\begin{equation}
	P_0(\vx_0) = \int d\vx_{1:N} \, \cHstar(\vx_{0:N}) = f(\vx_0) P_0(\vx_0) \int d\vx_{1:N} \, \cG(\vx_{1:N}\mid \vx_0)\, g(\vx_N),
\end{equation}
which implies
\begin{equation}
	f(\vx_0) = \frac{1}{\psi_0(\vx_0)}, \qquad \psi_0(\vx_0) := \int d\vx_{1:N} \, \cG(\vx_{1:N}\mid\vx_0)\, g(\vx_N).
\end{equation}
Similarly, enforcing the final marginal gives
\begin{equation}
	P_N(\vx_N) = g(\vx_N) \int d\vx_{0:N-1} \, \cHstar(\vx_{0:N-1}, \vx_N)
	= g(\vx_N)\,\bpsi_N(\vx_N),
\end{equation}
where the integral over the preceding slices reproduces the forward field $\bpsi_N$ defined below. These relations close the system, confirm the consistency of the forward and backward factorization, and recover the product identity $P_k(\vx_k) = \psi_k(\vx_k)\,\bpsi_k(\vx_k)$ at the endpoints.

\subsection{Recursion relations and the continuum limit}
We now derive the recursion relations for the auxiliary fields $\psi_k$ and $\bpsi_k$ that characterize the bridge at intermediate times. The backward field is the propagation of the terminal tilt,
\begin{equation} \label{eq:psi_def}
	\psi_k(\vx_k) = \int d\vx_{k+1:N} \, g(\vx_N) \prod_{s=k}^{N-1} G(\vx_{s+1} \mid \vx_s),
\end{equation}
and by the Chapman--Kolmogorov equation it satisfies the backward recursion
\begin{equation} \label{eq:chi-recursion}
	\psi_{k-1}(\vx_{k-1}) = \int d\vx_k \, \psi_k(\vx_k) \, G(\vx_k \mid \vx_{k-1}),
\end{equation}
with terminal condition $\psi_N(\vx_N) = g(\vx_N)$. Similarly, the forward field is
\begin{equation} \label{eq:bpsi_def}
	\bpsi_k(\vx_k) = \int d\vx_{0:k-1} \, f(\vx_0) P_0(\vx_0) \prod_{s=0}^{k-1} G(\vx_{s+1} \mid \vx_s),
\end{equation}
and satisfies the forward recursion
\begin{equation} \label{eq:chihat-recursion}
	\bpsi_{k+1}(\vx_{k+1}) = \int d\vx_k \, G(\vx_{k+1} \mid \vx_k) \, \bpsi_k(\vx_k),
\end{equation}
with initial condition $\bpsi_0(\vx_0) = f(\vx_0) P_0(\vx_0)$. The product identity $P_k(\vx_k) = \psi_k(\vx_k)\, \bpsi_k(\vx_k)$ follows directly from the factorized joint measure by integrating out all other time slices.

To obtain the continuum limit of these recursions, consider the backward recursion Eq.~\eqref{eq:chi-recursion}. As $\Dt \rightarrow 0$, the transition kernel $G(\vx_k \mid \vx_{k-1})$ is the short-time propagator of a diffusion with drift $\vv$ and diffusion coefficient $\gamma(t)$, and the standard Kramers--Moyal expansion \cite{Risken1996} yields the adjoint (backward) Kolmogorov equation
\begin{equation}
	\partial_t \psi + \vv \cdot \nabla \psi + \frac{\gamma(t)}{2} \nabla^2 \psi - V_G \psi = 0,
\end{equation}
with terminal condition $\psi(\vx, T) = g(\vx)$. The killing term $V_G$ appears if the reference process includes a potential allowing for killing, so that probability is not conserved. Similarly, the forward recursion Eq.~\eqref{eq:chihat-recursion} yields the forward Fokker--Planck equation
\begin{equation}
	\partial_t \bpsi + \nabla \cdot (\vv \bpsi) - \frac{\gamma(t)}{2} \nabla^2 \bpsi + V_G \bpsi = 0,
\end{equation}
with initial condition $\bpsi(\vx, 0) = f(\vx) P_0(\vx)$. These two equations are dual under time reversal and form the continuous-time backbone of the Schr\"odinger bridge.

\subsection{Short-time expansion of the tilted kernel}
The optimal transition kernel is
\begin{equation}  \label{eq:one_step_tilted_kernel}
	\Hstar(\vx_{k+1} \mid \vx_k) = \frac{\psi_{k+1}(\vx_{k+1})}{\psi_k(\vx_k)}\, G(\vx_{k+1} \mid \vx_k).
\end{equation}
To extract its short-time form, begin with the short-time reference kernel
\begin{equation} \label{eq:one_step_ref_kernel}
	G(\vx_{k+1} \mid \vx_k) = \frac{1}{(2\pi\gamma_k \Delta t)^{d/2}} \exp\left[ -\frac{\|\Delta\vx - \vv(\vx_k)\Delta t\|^2}{2\gamma_k \Delta t} - V_G(\vx_k)\Delta t \right],
\end{equation}
where $\Dx = \vx_{k+1} - \vx_k$ and $\gamma_k = \gamma(t_k)$. Under the reference measure the increment scales as $\Dx \sim \sqrt{\Dt}$, so the log-ratio of the $\psi$ fields must be expanded to first order in $\Dt$ and second order in $\Dx$,
\begin{equation}
	\begin{split}
		\ln \frac{\psi_{k+1}(\vx_{k+1})}{\psi_k(\vx_k)} ={}& \Delta\vx \cdot \nabla \ln \psi_k(\vx_k) + \frac{1}{2} \Delta\vx^a \Delta\vx^b \, \partial_a \partial_b \ln \psi_k(\vx_k) \\
		& + \Delta t \, \partial_t \ln \psi_k(\vx_k) + \mathcal O(\Delta t^{3/2}).
	\end{split}
\end{equation}
Substituting this expansion into Eq.~\eqref{eq:one_step_tilted_kernel} and combining with the Gaussian exponent of Eq.~\eqref{eq:one_step_ref_kernel}, the linear term in $\Dx$ completes the square against the reference exponent and produces the effective drift $\vv_* = \vv + \gamma_k \nabla \ln \psi_k$. The remaining terms of order $\Dt$ combine to give the residuals required by normalization. The tilted kernel is therefore, to leading order, a Gaussian diffusion with the same noise coefficient $\gamma_k$ and the effective drift
\begin{equation} \label{eq:effective_drift}
	\vv_{*}(\vx, t) = \vv(\vx, t) + \gamma(t)\, \nabla \ln \psi(\vx, t)
\end{equation}
in the continuum limit. Equation~\eqref{eq:effective_drift} is the central dynamical result of the Schr\"odinger bridge, since the optimal process is obtained by adding the gradient of the log-backward field to the reference drift.

\subsection{Conditional field construction and the posterior mixture}
In the score-matching sector of Sec.~\ref{sec:bridge}, we set $V_G = -\nabla\cdot\vv$ so that the forward field becomes constant, $\bpsi = 1$. The marginal is therefore $P_t = \psi_t$ and the effective drift is $\vv_* = \vv + \gamma \nabla \ln P_t$. It remains to express $\nabla \ln P_t$ in terms of the training data.

Because the equation governing $\psi$ is linear, its solution can be obtained by superposition. For a single data point $\vy$, define the conditional field
\begin{equation} \label{eq:conditional_chi}
	\psi^{(\vy)}(\vx, t) = G(\vy, T \mid \vx, t),
\end{equation}
the transition kernel from $(\vx, t)$ to the terminal point $\vy$ at time $T$. This field satisfies the same backward Kolmogorov equation as $\psi$, with terminal condition $\psi^{(\vy)}(\vx, T) = \delta^d(\vx - \vy)$, and by linearity the full field is the average over the data distribution,
\begin{equation}
	\psi(\vx, t) = \int d\vy \, P_{\rm data}(\vy) \, \psi^{(\vy)}(\vx, t).
\end{equation}

For the affine Gaussian diffusions used in practice, namely variance-preserving or variance-exploding schedules, the reference kernel is Gaussian, and reversing time via $\tau = T - t$ to pass from the generative orientation to the noising one gives
\begin{equation}
	\psi^{(\vy)}(\vx, t) = q_\tau(\vx \mid \vy) = \mathcal N\!\left(\vx; \alpha_\tau \vy,\; \sigma_\tau^2 \mathbf I\right),
\end{equation}
where $\alpha_\tau$ and $\sigma_\tau$ are the signal and noise schedules of the noising process. Suppressing the orientation change in the notation from here on, the marginal is consequently
\begin{equation}
	P_t(\vx) = \int d\vy \, P_{\rm data}(\vy) \, q_t(\vx \mid \vy),
\end{equation}
which is precisely the forward noising marginal.

The conditional bridge kernel for a fixed data point $\vy$ follows from the tilted kernel formula,
\begin{equation}
	H^{(\vy)}(\vx_{k+1} \mid \vx_k) = \frac{\psi^{(\vy)}_{k+1}(\vx_{k+1})}{\psi^{(\vy)}_{k}(\vx_k)}\, G(\vx_{k+1} \mid \vx_k),
\end{equation}
and by the Markov property of the noising chain this simplifies to the posterior of the noising chain conditioned on the clean sample,
\begin{equation}
	H^{(\vy)}(\vx_{k+1} \mid \vx_k) = q(\vx_{k+1} \mid \vx_k, \vy) = \frac{q(\vx_k \mid \vx_{k+1}) \, q(\vx_{k+1} \mid \vy)}{q(\vx_k \mid \vy)}.
\end{equation}
Averaging over the data distribution with the posterior weight
\begin{equation}
	p(\vy \mid \vx_k) = \frac{P_{\rm data}(\vy) \, q(\vx_k \mid \vy)}{P_k(\vx_k)}
\end{equation}
yields the full optimal reverse kernel,
\begin{equation}
	\Hstar(\vx_{k+1} \mid \vx_k) = \int d\vy \, p(\vy \mid \vx_k) \, q(\vx_{k+1} \mid \vx_k, \vy).
\end{equation}
This is the posterior-weighted mixture representation of the optimal bridge. The generative process averages over the posterior of the clean data point given the current noisy observation and then steps according to the conditional noising posterior.

Finally, the same construction expresses the marginal score as a posterior expectation,
\begin{equation}
	\nabla \ln P_t(\vx) = \frac{\int d\vy \, P_{\rm data}(\vy) \, q_t(\vx \mid \vy) \, \nabla \ln q_t(\vx \mid \vy)}{P_t(\vx)} = \mathbb{E}_{p(\vy \mid \vx)} \left[ \nabla \ln q_t(\vx \mid \vy) \right].
\end{equation}
This is the identity underlying the simulation-free score-matching objective, in which the unknown marginal score is replaced by a tractable posterior expectation estimated in practice by sampling a clean data point and adding noise.

\subsection{Short-time KL divergence between diffusion kernels}
\label{app:short-time-expansion}

We derive the per-step KL divergence between two diffusion processes that share the same diffusion coefficient but differ in drift, the computation underlying the excess action of Sec.~\ref{subsec:excess}. Consider two transition kernels over a small interval $\Delta t$ with drifts $\vv_1(\vx_k)$, $\vv_2(\vx_k)$ and common diffusion coefficient $\gamma_k$,
\begin{equation}
	H_i(\vx_{k+1} \mid \vx_k) = \frac{1}{(2\pi\gamma_k \Delta t)^{d/2}} \exp\left[ -\frac{\|\Delta\vx - \vv_i(\vx_k)\Delta t\|^2}{2\gamma_k \Delta t} \right], \quad i = 1, 2,
\end{equation}
where $\Dx = \vx_{k+1} - \vx_k$. The drifts are evaluated at the initial point $\vx_k$, and higher-order terms in the drift do not affect the leading-order divergence. The per-step KL divergence is
\begin{equation}
	\begin{split}
		D_{KL}\bigl(H^{(1)}(\cdot\mid\vx_k) \,\|\, H^{(2)}(\cdot\mid\vx_k)\bigr)
		= \int \dx_k \, \dx_{k+1} \, P^{(1)}_k (\vx_k) \, H^{(1)}(\vx_{k+1}\mid\vx_k) & \\
		\times \ln \frac{H^{(1)}(\vx_{k+1}\mid\vx_k)}{H^{(2)}(\vx_{k+1}\mid\vx_k)} & ,
	\end{split}
	\label{eq:per_step_kl}
\end{equation}
where $P^{(1)}_k$ is the marginal under the first process, or any measure with the same support, since the leading-order result is independent of the weighting. The log-ratio simplifies because the normalization factors cancel,
\begin{equation}
	\ln \frac{H_1}{H_2} = \frac{1}{\gamma_k} \Delta\vx \cdot (\vv_1 - \vv_2) + \frac{\Delta t}{2\gamma_k} \left( \|\vv_2\|^2 - \|\vv_1\|^2 \right).
\end{equation}
Under $H_1$ the increment is distributed as
\begin{equation}
	\Delta\vx = \vv_1(\vx_k) \Delta t + \sqrt{\gamma_k \Delta t} \, \boldsymbol{\xi}, \qquad \boldsymbol{\xi} \sim \mathcal N(0, \mathbf I),
\end{equation}
and substituting into the log-ratio and averaging over $\boldsymbol{\xi}$ gives
\begin{align}
	\E_{H_1}\left[ \ln \frac{H_1}{H_2} \right] &= \frac{\Delta t}{\gamma_k}\, \vv_1 \cdot (\vv_1 - \vv_2) + \frac{\Delta t}{2\gamma_k} \left( \|\vv_2\|^2 - \|\vv_1\|^2 \right) \nonumber \\ &= \frac{\Delta t}{2\gamma_k} \|\vv_1 - \vv_2\|^2,
\end{align}
where the term proportional to $\E[\boldsymbol{\xi}]$ vanishes. This is exact for the expectation of the log-ratio under $H_1$, because the log-ratio is linear in $\Delta\vx$ plus a deterministic constant. No higher-order corrections in $\Dt$ arise beyond the time-dependence of the drifts over the interval, which contributes $\mathcal O(\Dt^{3/2})$ or higher. Therefore
\begin{equation}
	D_{KL}\bigl(H_1(\cdot\mid \vx_k) \,\|\, H_2(\cdot \mid \vx_k)\bigr) = \frac{\Delta t}{2\gamma_k} \|\vv_1(\vx_k) - \vv_2(\vx_k)\|^2 + \mathcal O(\Delta t^{3/2}).
\end{equation}
Integrating over the initial state $\vx_k$ with weight $P^{(1)}_k(\vx_k)$ and taking the continuum limit yields the integral form used in the main text,
\begin{equation}
	\Delta \mathcal S = \frac{1}{2} \int_0^T \frac{dt}{\gamma(t)} \, \E_{\vx \sim P_t} \left[ \|\vv_H(\vx, t) - \vv_*(\vx, t)\|^2 \right] + \mathcal O(\Delta t),
\end{equation}
where the error term vanishes as the discretization is refined. This establishes the connection between the path-space KL divergence and the squared drift difference underlying the score-matching objective.

\section{Details of the Latent-Geometry Construction}
\label{app:latent}

\renewcommand{\theequation}{\thesection.\arabic{equation}}
\setcounter{equation}{0}

This appendix supplies the derivations for Sec.~\ref{sec:fisher}. We compute the ambient second variation and establish the endpoint measurability of the bridge tangent space. We then show that exact path conditioning is singular, work out the exponential-family form of the Gaussian corruption, and carry out the pullback with its endpoint split. The final subsections give the Gaussian evaluation and the radial information identity.

\subsection{The ambient second variation}

Let $\cG$ be a $\sigma$-finite reference path measure and
\begin{equation}
	\Lambda[u,v] = \ln \int \dx_{0:N}\, \cG(\vx_{0:N})\, e^{u(\vx_0)+v(\vx_N)},
	\qquad
	\cD = \{(u,v):\Lambda[u,v]<\infty\}.
\end{equation}
Write $\rho^{(u,v)}$ for the joint law of $(\vx_0,\vx_N)$ under $\cH^{(u,v)}$ of Eq.~\eqref{eq:tilt_family}, and let $S = \delta u(\vx_0)+\delta v(\vx_N)$ for an admissible direction $(\delta u,\delta v)$.

Consider $\lambda \mapsto \Lambda[u+\lambda\,\delta u,\ v+\lambda\,\delta v]$. By construction
\begin{equation}
	\Lambda[u+\lambda\delta u, v+\lambda\delta v] = \Lambda[u,v] + \ln \E_{\rho^{(u,v)}}\big[e^{\lambda S}\big],
\end{equation}
since tilting by $\lambda S$ reweights the endpoint law by exactly $e^{\lambda S}$. The second term is the cumulant-generating function of $S$ under $\rho^{(u,v)}$, and differentiating twice at $\lambda=0$ gives
\begin{equation}
	\frac{d}{d\lambda}\Big|_0 = \E_\rho[S],
	\qquad
	\frac{d^2}{d\lambda^2}\Big|_0 = \Var_\rho[S],
\end{equation}
which is Lemma~\ref{lem:ambient_fisher}. Differentiation under the integral is justified for $(u,v)$ interior to $\cD$ by the standard exponential-integrability argument for exponential families \cite{LDBrown1986, BarndorffNielsen1978}, because on the interior $\E_\cG[e^{(u+\lambda\delta u)(\vx_0)+(v+\lambda\delta v)(\vx_N)}]$ is finite on a neighbourhood of $\lambda = 0$ and analytic there.

Nothing beyond $\sigma$-finiteness and exponential integrability was used. In particular, $\cG$ need not be a probability measure, need not be Markov, and no finite-dimensional parameterization is involved.

\subsection{Endpoint measurability of the tangent space}

The log-likelihood ratio between two nearby members of the tilt family is, directly from Eq.~\eqref{eq:tilt_family},
\begin{equation}
	\ln\frac{\cH^{(u+\delta u,\,v+\delta v)}}{\cH^{(u,v)}}(\vx_{0:N})
	= \delta u(\vx_0) + \delta v(\vx_N) - \delta\Lambda,
\end{equation}
so the centred score is
\begin{equation} \label{eq:score_endpoint}
	\delta\ell = \delta u(\vx_0) + \delta v(\vx_N) - \E_\rho\big[\delta u(\vx_0)+\delta v(\vx_N)\big],
\end{equation}
a function of the endpoints alone. Every tangent vector of the bridge family is therefore represented by an endpoint-measurable score. The whole local statistical experiment, namely the collection of likelihood ratios distinguishing infinitesimally separated bridges, is measurable with respect to $\sigma(\vx_0,\vx_N)$. Lemma~\ref{lem:ambient_fisher} is the second moment of Eq.~\eqref{eq:score_endpoint}.

This is the precise sense in which passing from the path to the endpoint pair loses nothing relevant. By the monotonicity of Fisher information under Markov kernels, any reduction of the path variable can only decrease the Fisher form, with equality for a sufficient reduction \cite{Chentsov1982, AyJostLeSchwachhofer2017}. Equation~\eqref{eq:score_endpoint} exhibits the endpoint pair as sufficient for the tangent experiment, so the reduction is lossless at this order.

\subsection{Singularity of exact path conditioning}

Let $\cG$ be a nondegenerate diffusion reference and consider $\cH(\cdot\mid X_t=\vx)$. The event $\{X_t=\vx\}$ has probability zero, and the conditional measures for distinct $\vx$ are carried by the disjoint path sets $\{\omega: \omega(t)=\vx\}$. Hence for $\vx\neq\vx'$ the two measures are mutually singular and
\begin{equation}
	D_{KL}\big(\cH(\cdot\mid X_t=\vx)\,\big\|\,\cH(\cdot\mid X_t=\vx')\big) = +\infty .
\end{equation}
A Fisher metric requires $D_{KL}(P_{\theta+d\theta}\|P_\theta) = \tfrac12 g_{ij}d\theta^i d\theta^j + O(|d\theta|^3)$, but no such expansion exists here at any order. The obstruction is not specific to Schr\"odinger bridges but is generic for exact point conditioning of continuous-path measures, and it is the reason the reduction Eq.~\eqref{eq:endpoint_immersion} conditions the endpoint experiment rather than the path.

\subsection{Exponential-family form of the Gaussian corruption}

With $q(\vx,t\mid\vx_N) = \cN(\vx;\alpha_t\vx_N,\sigma_t^2\mathbf I)$,
\begin{equation}
	\ln q = -\frac{\sqnorm{\vx - \alpha_t\vx_N}}{2\sigma_t^2} + \text{const}
	= \frac{\alpha_t}{\sigma_t^2}\,\vx\cdot\vx_N \;-\; \frac{\alpha_t^2}{2\sigma_t^2}\sqnorm{\vx_N} \;-\; \frac{\sqnorm{\vx}}{2\sigma_t^2} + \text{const},
\end{equation}
which is Eq.~\eqref{eq:ref_exp_split} with
\begin{equation} \label{eq:app_A_explicit}
	\bs(\vx_N) = \Big(\vx_N, -\tfrac12\sqnorm{\vx_N}\Big),
	\qquad
	A(\vx,t) = \Big(\frac{\alpha_t}{\sigma_t^2}\vx,\ \frac{\alpha_t^2}{\sigma_t^2}\Big),
\end{equation}
and $B$, $C$ collecting the terms depending on $\vx$ alone and on $\vx_N$ alone. The coefficient of the second-moment statistic is the signal-to-noise ratio $\mathrm{SNR}_t = \alpha_t^2/\sigma_t^2$, so with $\tilde\vx = \vx/\alpha_t$,
\begin{equation}
	A(\vx,t) = \mathrm{SNR}_t\cdot(\tilde\vx,1),
\end{equation}
which is Eq.~\eqref{eq:cone}. At fixed $\tilde\vx$ the time derivative is $\partial_t A = (\partial_t \mathrm{SNR}_t)\,(\tilde\vx,1) \parallel A$, so the temporal direction of latent spacetime is radial in natural-parameter space and the spatial directions are angular.

For an \emph{anisotropic} Gaussian corruption with covariance $\Sigma_t$, the term $-\tfrac12\vx_N^\top\Sigma_t^{-1}\vx_N$ does not reduce to a multiple of $\sqnorm{\vx_N}$, and the sufficient statistic must carry the full symmetric tensor $-\tfrac12\vx_N\otimes\vx_N$. The natural-parameter space then has $m = d + d(d+1)/2$ dimensions while the latent manifold still has $d+1$. The image of $\iota_{\rm lat}$ therefore acquires codimension, the immersion is no longer onto an open set, and the latent family becomes a curved exponential family with nonvanishing embedding curvature $\Pi^{\perp}[\partial_\mu\partial_\nu A]$. Isotropy is exactly the condition $m = d+1$.

\subsection{Pullback and endpoint split}

The conditional endpoint law Eq.~\eqref{eq:p_g} is an exponential family with carrier $\nu$ and natural parameter $A(\vx,t)$, so Lemma~\ref{lem:ambient_fisher} applies to it verbatim and its Fisher form in natural coordinates is $\Cov_p[\bs_a,\bs_b]$. Composing with $\iota_{\rm lat}$ and using the chain rule gives Eq.~\eqref{eq:factorization},
\begin{equation}
	g_{\mu\nu} = \partial_\mu A^a\,\partial_\nu A^b\,\Cov_{p}[\bs_a,\bs_b].
\end{equation}

For the split, the bridge is Markov, so conditioning on $X_t=\vx$ makes past and future independent,
\begin{equation}
	p(\vx_0,\vx_N\mid X_t=\vx) = p(\vx_0\mid X_t=\vx)\,p(\vx_N\mid X_t=\vx).
\end{equation}
The score of the joint conditional with respect to $\vx$ is then the sum of the two individual scores, and because the two factors are independent under the conditional measure the cross-covariance vanishes, giving Eq.~\eqref{eq:tangent_chain} with
\begin{align}
	g^{\rm future}_{ij} &= \Cov_{p(\vx_N\mid\vx,t)}\big[\partial_i \ln q(\vx,t\mid\vx_N),\ \partial_j \ln q(\vx,t\mid\vx_N)\big], \\
	g^{\rm past}_{ij} &= \Cov_{p(\vx_0\mid\vx,t)}\big[\partial_i \ln G(\vx,t\mid\vx_0),\ \partial_j \ln G(\vx,t\mid\vx_0)\big].
\end{align}
No such cancellation occurs in Eq.~\eqref{eq:three_terms}, where the endpoints are not conditioned and their covariance is generically nonzero.

\subsection{Gaussian evaluation}

From Eq.~\eqref{eq:gaussian_encoder},
\begin{equation}
	\partial_i \ln q(\vx,t\mid\vx_N) = -\frac{x_i - \alpha_t x_{N,i}}{\sigma_t^2},
\end{equation}
whose only $\vx_N$-dependence is through $\alpha_t x_{N,i}/\sigma_t^2$. Hence
\begin{equation} \label{eq:app_gfuture}
	g^{\rm future}_{ij}(\vx,t) = \frac{\alpha_t^2}{\sigma_t^4}\,\Cov\big[x_N^i, x_N^j \mid X_t=\vx\big].
\end{equation}
To convert this into Hessian form, write the marginal as an integral over the data distribution, $P_t(\vx) = \int d\vx_N\, P_{\rm data}(\vx_N)\, q(\vx,t\mid\vx_N)$. Differentiating once,
\begin{equation}
	\partial_i \ln P_t(\vx) = \E\big[\partial_i \ln q \mid \vx\big] = \frac{\alpha_t \E[x_N^i\mid\vx] - x_i}{\sigma_t^2},
\end{equation}
which is the Tweedie relation \cite{Robbins1992, Efron2011TweediesFA}. Differentiating a second time and using $\partial_j \E[x_N^i\mid\vx] = (\alpha_t/\sigma_t^2)\Cov[x_N^i,x_N^j\mid\vx]$,
\begin{equation}
	\partial_i\partial_j \ln P_t(\vx) = -\frac{\delta_{ij}}{\sigma_t^2} + \frac{\alpha_t^2}{\sigma_t^4}\Cov\big[x_N^i,x_N^j\mid\vx\big].
\end{equation}
Combining with Eq.~\eqref{eq:app_gfuture} gives Eq.~\eqref{eq:g_chi},
\begin{equation}
	g^{\rm future}_{ij}(\vx,t) = \frac{1}{\sigma_t^2}\delta_{ij} + \partial_i\partial_j \ln P_t(\vx).
\end{equation}
Comparison with Theorem~\ref{thm:floor_fisher} is immediate. The theorem's metric is $\E[\partial_i\ln p(\vx_N\mid\vx,t)\,\partial_j\ln p(\vx_N\mid\vx,t)]$, and by Eq.~\eqref{eq:posterior_score} the posterior score is $\partial_i \ln q - \partial_i \ln P_t$. Its second moment under the posterior is the covariance of $\partial_i\ln q$, which is exactly $g^{\rm future}$.

\subsection{The radial information identity}

Let $p_\theta(\vx_N)\propto \nu(\vx_N)e^{\theta\cdot\bs}$ and $K(\theta) = D_{KL}(p_\theta\|\nu)$. Using $\Lambda(\theta) = \ln\int\nu\, e^{\theta\cdot\bs}$ and $\eta = \nabla\Lambda = \E_{p_\theta}[\bs]$,
\begin{equation}
	K(\theta) = \E_{p_\theta}\big[\theta\cdot\bs - \Lambda(\theta)\big] = \theta\cdot\eta(\theta) - \Lambda(\theta),
\end{equation}
so that
\begin{equation}
	\nabla_\theta K = \eta + \theta^\top\nabla^2\Lambda - \nabla\Lambda = g(\theta)\,\theta,
	\qquad g = \nabla^2\Lambda = \Cov_{p_\theta}[\bs].
\end{equation}
Restricting to the ray $\theta = r\,n$ with $n$ fixed and normalized,
\begin{equation}
	\frac{dK}{dr} = n\cdot\nabla_\theta K = r\, n^\top g\, n = r\,g_{rr},
\end{equation}
which is Eq.~\eqref{eq:radial_info}. Since by Eq.~\eqref{eq:cone} the radial coordinate is $\mathrm{SNR}_t$ and the angular coordinate is the rescaled latent, the corruption moves along the ray and $K$ decreases monotonically as $\mathrm{SNR}_t$ decreases, at a rate set by the radial component of the Fisher metric. At $r=0$ the conditional endpoint law coincides with the carrier and the latent retains nothing about the data.

\section{Proofs of the Information-Flow Identities}
\label{app:infoflow}

\renewcommand{\theequation}{\thesection.\arabic{equation}}
\setcounter{equation}{0}

\begin{proof}[Proof of Lemma~\ref{lem:fisher_gap}]
	By Eq.~\eqref{eq:posterior_score} and the centering identity, for fixed $\vx$
	\begin{align*}
		\E_{\vy\mid\vx}\big\|\nabla_{\!\vx}\ln p(\vy\mid\vx,t)\big\|^2
		&= \E_{\vy\mid\vx}\big\|\nabla_{\!\vx}\ln q(\vx,t\mid\vy)\big\|^2
		- 2\,\nabla\!\ln P_t\cdot \E_{\vy\mid\vx}\big[\nabla_{\!\vx}\ln q\big] \\
		&\quad + \big\|\nabla\!\ln P_t\big\|^2 \\
		&= \E_{\vy\mid\vx}\big\|\nabla_{\!\vx}\ln q(\vx,t\mid\vy)\big\|^2 - \big\|\nabla\!\ln P_t(\vx)\big\|^2 .
	\end{align*}
	Averaging over $\vx\sim P_t$ and using the tower property,
	$\E_{P_t}\E_{\vy\mid\vx}\|\nabla\ln q\|^2 = \E_{\vy}\E_{\vx\mid\vy}\|\nabla\ln q(\cdot\mid\vy)\|^2 = \E_\vy J(q_t(\cdot\mid\vy))$.
\end{proof}

\begin{proof}[Proof of Lemma~\ref{lem:debruijn}]
	Any solution of the Fokker--Planck equation $\partial_t p = -\nabla\!\cdot\!(f p) + \tfrac{D}{2}\nabla^2 p$ satisfies $\tfrac{d}{dt}h(p_t) = \E_{p_t}[\nabla\!\cdot\! f] + \tfrac{D}{2}J(p_t)$, by two integrations by parts. Apply this to $P_t$ and to $q_t(\cdot\mid\vy)$ for each $\vy$, and average the latter over $\vy$. The drift terms then agree by the tower property and cancel in $I(\vy;\vx_t) = h(\vx_t) - h(\vx_t\mid\vy)$. Eq.~\eqref{eq:fisher_gap} converts the result into the metric trace.
\end{proof}

\begin{proof}[Proof of Theorem~\ref{thm:floor_information}]
	Theorem~\ref{thm:floor_fisher} with the Gaussian evaluation Eq.~\eqref{eq:g_chi} gives the integrand $\gamma(t)(\alpha_t^2/\sigma_t^4)\operatorname{tr}\Cov[\vy\mid\vx]$. By Lemma~\ref{lem:schedule} the prefactor is $d\lambda/dt$, so the $t$-integral becomes $\frac12\int \mathrm{MMSE}\,d\lambda$. The second equality is the I--MMSE relation of Guo, Shamai and Verd\'u \cite{GuoShamaiVerdu2005}, $dI(\vy;\vx_\lambda)/d\lambda = \tfrac12\mathrm{MMSE}(\lambda)$. The third holds because $h(\vx_\lambda\mid\vy) = h(\vz)$ is independent of $\lambda$, so $I = h(\vx_\lambda) - h(\vz)$ and the constant cancels in the difference.
\end{proof}

\section{Numerical and Experimental Details}
\label{app:numerics}

\renewcommand{\theequation}{\thesection.\arabic{equation}}
\setcounter{equation}{0}

This appendix collects the numerical verifications quoted in the main text. All experiments are analytic or low-dimensional by design, so that every quantity can be computed exactly or by quadrature. They establish identities rather than performance.

\subsection{de Bruijn identity for a nonlinear drift}
For the nonlinear corruption drift $f(x) = -a x^3$, we solved the Fokker--Planck equation directly and evaluated both sides of Eq.~\eqref{eq:debruijn}. The two sides agree with a median relative error of $1.7\times10^{-3}$ point-wise in time, and the integrated form agrees to $0.4\%$.

\subsection{The schedule identity}
For the linear interpolant $\alpha_t=t$, $\sigma_t=1-t$, the bridge coefficient of Eq.~\eqref{eq:bridge_coefficient} is $\gamma(t) = 2\sigma_t^2\,\frac{d}{dt}\ln(\alpha_t/\sigma_t) = 2(1-t)/t$, so the left-hand side of Eq.~\eqref{eq:schedule_identity} is $2\frac{1-t}{t}\frac{t^2}{(1-t)^4} = \frac{2t}{(1-t)^3}$, which equals $\frac{d}{dt}\frac{t^2}{(1-t)^2} = \frac{d\lambda_t}{dt}$. The trigonometric and polynomial interpolants follow by the same computation, and we have verified Eq.~\eqref{eq:schedule_identity} numerically to a relative error of $3\times10^{-6}$ for four distinct interpolants.

\subsection{The floor--entropy identity}
Figure~\ref{fig:floor} evaluates Theorem~\ref{thm:floor_information} on a two-mode Gaussian mixture for which the differential entropy $h(\vx_\lambda)$ is computable by quadrature. The accumulated floor $\frac12\int\mathrm{MMSE}\,d\lambda$ and the entropy change $\Delta h$ agree to a relative error of $2\times10^{-5}$.

\subsection{High-SNR slopes}
Figure~\ref{fig:dim} fits the divergence rate of the floor for Gaussian data of rank $k$ embedded in $\R^{10}$. The fitted slopes match $k/2$ to three decimals, confirming Eq.~\eqref{eq:mmse_gauss}.

\subsection{Ranking inversion}
The example of Sec.~\ref{subsec:losscompare} uses the analytic two-mode mixture. The two SNR ranges are $A$ (narrow) and $B$ (wide). The two models are exact scores scaled by $c = 0.97$ (better) and $c = 0.90$ (worse). Raw losses are $1.0490$ and $1.1267$ on $A$ and $1.5455$ and $1.8601$ on $B$. The floors, computed from Theorem~\ref{thm:floor_fisher} as expectations of the posterior covariance, are $1.0414$ and $1.5143$. The floor-subtracted excesses are $0.0077 < 0.0854$ on $A$ and $0.0311 < 0.3458$ on $B$.

\subsection{Third-cumulant localization}
The measurements of Sec.~\ref{subsec:third_order} evaluate both terms of $d^2\eta/d\lambda^2$ in Eq.~\eqref{eq:cumulant_expansion} along exactly integrated probability-flow trajectories on the same mixture. The share of $T[u,u]$ in $\|d^2\eta/d\lambda^2\|$ is $0.05$ at $t=0.10$, $0.85$ at $t=0.30$, $1.19$ at $t=0.45$ (values above one indicating partial cancellation between the two terms), and $0.11$ at $t=0.65$, with a median of $0.47$ across trajectories. The step-placement comparison of Fig.~\ref{fig:consequences} uses an explicit first-order Euler solver, with endpoint error measured against the exactly integrated probability-flow ODE.

\subsection{The discrete floor}
For $n=3$ tokens over an alphabet of size $3$ with total correlation $0.46$ nats, Eq.~\eqref{eq:discrete_floor} reproduces the data entropy $H(y_{1:n}) = 2.58757$ to a relative error of $8\times10^{-8}$. Integrating in $t$ against the weight $-\dot m_t/m_t$ returns the same value for the masking schedules $m = 1-t$, $\cos(\pi t/2)$, $(1-t)^3$, and $\sqrt{1-t}$, confirming schedule independence.

\end{document}